\documentclass{article}

\PassOptionsToPackage{numbers, compress}{natbib}

    \usepackage[final]{neurips_2024}

\usepackage[utf8]{inputenc} 
\usepackage[T1]{fontenc}    
\usepackage{url}            
\usepackage{booktabs}       
\usepackage{amsfonts}       
\usepackage{nicefrac}       
\usepackage{microtype}      
\usepackage{xcolor}         

\usepackage{geometry}
\usepackage[colorlinks, linkcolor=orange, anchorcolor=blue, citecolor=blue]{hyperref}
\usepackage{enumitem}
\usepackage{algorithm}
\usepackage{algorithmic}
\usepackage{amsmath}
\usepackage{amssymb}
\usepackage{mathtools}
\usepackage{amsthm}
\usepackage{comment}
\usepackage{graphicx}

\usepackage[capitalize,noabbrev]{cleveref}

\usepackage{wrapfig}
\usepackage{tabularx,booktabs}
\usepackage{subfigure}
\newcolumntype{M}[1]{>{\centering\arraybackslash}m{#1}}

\newcommand{\pscore}{\sbb_{\parallel}}
\newcommand{\oscore}{\sbb_{\perp}}

\newcommand*\diff{\mathop{}\!\mathrm{d}}

\newcommand{\bXb}{X^{\leftarrow}}

\usepackage{def}
\usepackage{xcolor}
\definecolor{darktangerine}{rgb}{1.0, 0.66, 0.07}
\def\HiLi{\leavevmode\rlap{\hbox to .9\hsize{\color{yellow!50}\leaders\hrule height .8\baselineskip depth .4ex\hfill}}}

\definecolor{ForestGreen}{RGB}{34,139,34}

\newcommand{\rbra}[1]{\left( #1 \right)}
\newcommand{\sbra}[1]{\left[ #1 \right]}
\newcommand{\cbra}[1]{\left\{ #1 \right\}}
\newcommand{\pbra}[1]{\left< #1 \right>}
\newcommand{\fs}[2]{\displaystyle\frac{#1}{#2}}
\newcommand{\dabs}[1]{\left \lVert #1 \right\lVert}
\newcommand*{\dif}{\mathop{}\!\mathrm{d}}

\usepackage{titletoc}
\newcommand\DoToC{%
  \startcontents
  \printcontents{}{2}{\textbf{Contents}\vskip3pt\hrule\vskip5pt}
  \vskip3pt\hrule\vskip5pt
}

\newcounter{module}
\makeatletter
\newenvironment{module}[1][htb]{%
  \let\c@algorithm\c@module
  \renewcommand{\ALG@name}{Module}
  \begin{algorithm}[#1]%
  }{\end{algorithm}
}

\usepackage{setspace}
\AtBeginDocument{%
  \addtolength\abovedisplayskip{-0.35\baselineskip}%
  \addtolength\belowdisplayskip{-0.35\baselineskip}%
  \addtolength\abovedisplayshortskip{-0.35\baselineskip}%
  \addtolength\belowdisplayshortskip{-0.35\baselineskip}%
}

\usepackage{titlesec}
 \titlespacing*{\section}{0pt}{-0.03\baselineskip}{-0.025\baselineskip}
\titlespacing*{\subsection}{0pt}{-0.02\baselineskip}{-0.015\baselineskip}
\titlespacing*{\subsubsection}{0pt}{-0.025\baselineskip}{-0.025\baselineskip}

\title{Gradient Guidance for Diffusion Models: \\ An Optimization Perspective}

\author{
Yingqing Guo \thanks{Equal contribution. Department of Electrical and Computer Engineering, Princeton University. Authors’ emails are: \texttt{\{yg6736, huiyuan, yy1325, minshuochen, mengdiw\}@princeton.edu}. }
\And Hui Yuan \footnotemark[1] \And Yukang Yang \And Minshuo Chen \And Mengdi Wang
\AND Princeton University
}

\begin{document}

\maketitle

\begin{abstract}
  Diffusion models have demonstrated empirical successes in various applications and can be adapted to task-specific needs via guidance. This paper studies a form of gradient guidance for adapting a pre-trained diffusion model towards optimizing user-specified objectives. We establish a mathematical framework for guided diffusion to systematically study its optimization theory and algorithmic design. Our theoretical analysis spots a strong link between guided diffusion models and optimization: gradient-guided diffusion models are essentially sampling solutions to a regularized optimization problem, where the regularization is imposed by the pre-training data. As for guidance design, directly bringing in the gradient of an external objective function as guidance would {\it jeopardize} the structure in generated samples. We investigate a modified form of gradient guidance based on a forward prediction loss, which leverages the information in pre-trained score functions and provably preserves the latent structure. We further consider an iteratively fine-tuned version of gradient-guided diffusion where guidance and score network are both updated with newly generated samples. This process mimics a first-order optimization iteration in expectation, for which we proved $\tilde{\mathcal{O}}(1/K)$ convergence rate to the global optimum when the objective function is concave. Our code will be released at \href{https://github.com/yukang123/GGDMOptim.git}{https://github.com/yukang123/GGDMOptim.git}.
\end{abstract}

\section{Introduction}\label{sec:intro}
Diffusion models have emerged as a significant advancement in the field of generative artificial intelligence, offering state-of-the-art performance in image generation \citep{song2019generative, song2020denoising, dhariwal2021diffusion}. 
These models operate by gradually transforming a random noise into a structured output, utilizing the score function learned from data.
One of the key advantages of diffusion models is their flexibility which allows controlled sample generation for task-specific interest, excelling diffusion models in a wide range of applications, such as content creation, sequential decision making, protein engineering \citep{kong2020diffwave, ajay2022conditional, gruver2023protein}. 

Controlling the generation of large generative models stands at the forefront of AI. Guidance and fine-tuning are two most prevalent approaches for controlling the generation of diffusion models. Unlike fine-tuning which changes the weights of pre-trained models, guidance mechanism enables a more directed and flexible control. 
Adding gradient-based guidance during inference was pioneered by classifier guidance \citep{song2020score, dhariwal2021diffusion}, which involves training a time-dependent classifier. Diffusion Posterior Sampling (DPS) \citep{chung2022diffusion} introduced a fully training-free form of gradient-based guidance, which removes the dependence on time. This method has since been explored in numerous empirical studies \cite{chung2022improving,kawar2022denoising, lugmayr2022repaint,wang2022zero,song2023loss, yu2023freedom, bansal2023universal}. However, despite these empirical successes, significant gaps remain in the theoretical understanding and guarantees of gradient-based guidance in diffusion models.

\paragraph{Problem and Challenges} Suppose we have a pre-trained diffusion model that can generate new samples faithfully, maintaining the latent structure of data. We study the problem of adapting this diffusion model to generate new samples that optimize task-specific objectives, while {\it maintaining the learned structure} in new samples. This problem has a strong connection to classic optimization, guided diffusion offers new possibilities to optimize complex design variables such as images, videos, proteins, and genomes \citep{black2023training, watson2023novo, liu2024sora} in a generative fashion. More comprehensive exposure to this middle ground can be found in recent surveys \citep{yang2023diffusion,chen2024overview,guo2023diffusion}.

Given the optimization nature of this problem, it's critical to answer the following theoretical questions from an optimization perspective:
 \textit{(i)} Why doesn't simply applying the gradient of the objective function w.r.t. the noised sample work? 
 \textit{(ii)} How to add a guidance signal to improve the target objective without compromising the sample quality?  
 \textit{(iii)} Can one guarantee the optimization properties of new samples generated by guided diffusion?  \textit{(iv)} What are the limits of adaptability in these guided models? 

\begin{figure}[!htb]
    \centering
    \vspace{-5pt}
    \includegraphics[width = 0.85\textwidth]{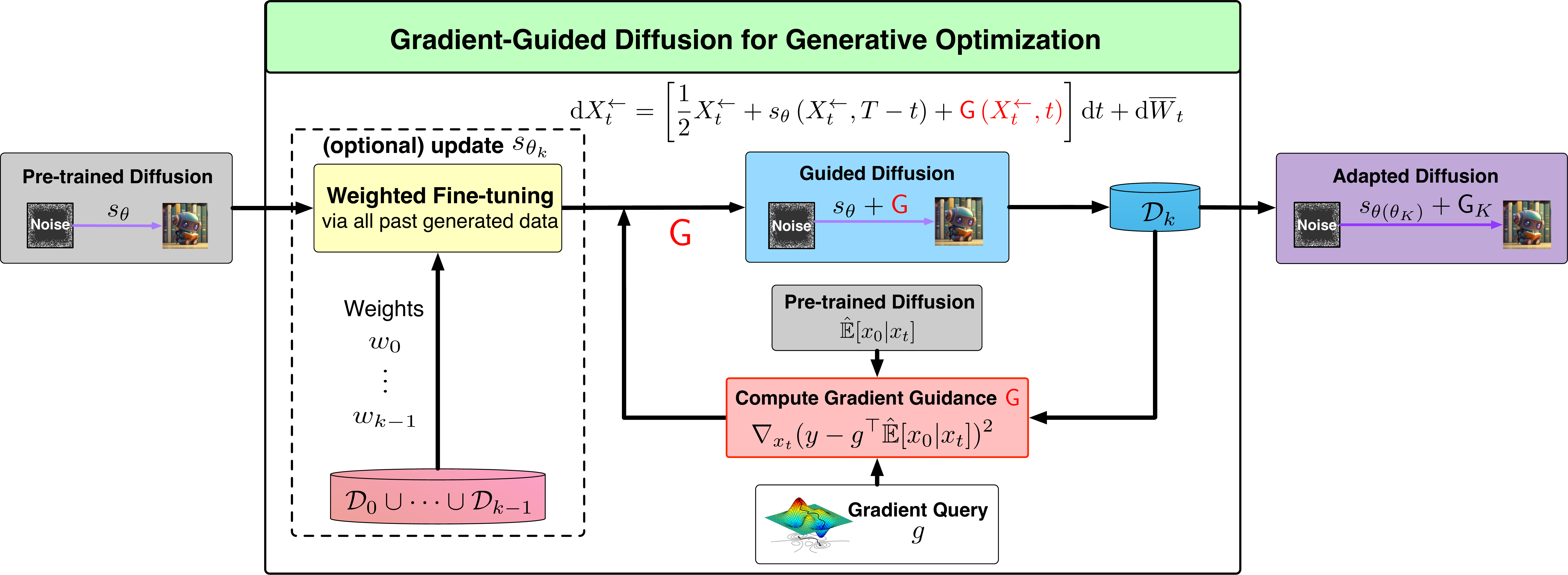}
    \caption{\small \textbf{Gradient-guided diffusion model for generative optimization, with or without adaptive finetuning.} A pre-trained diffusion model is guided with an additional gradient signal from an external objectives function towards generating near-optimal solutions. 
    }
    \vspace{-2pt}
    \label{fig:framework}
\end{figure}

\vspace{-10pt}
\paragraph{Scope and Contribution.} In this paper, we investigate guided diffusion from an optimization perspective. To answer the four questions above, we propose an algorithmic framework, see Figure \ref{fig:framework} for an illustration. Our main contributions are summarized as follows:

\vspace{-3pt}
\noindent $\bullet$ \textbf{Study structure-preserving guidance.} We study the design of guidance under structural data distribution belonging to a latent low-dimensional subspace (\cref{asmp:subspace_data}). We diagnose the failure of naive gradient guidance and study the theoretical aspects of guidance based on forward prediction loss (Definition~\ref{def:gradient_guidance+}), which provably preserves any low-dimensional subspace structure (\cref{thm:faith_grad_guidance}). 

\vspace{-3pt}
\noindent $\bullet$  \textbf{Establish a mathematical framework of guided diffusion.} We build a mathematical framework for guided diffusion, which facilitates algorithm analysis and theory establishment. We propose and analyze an iterative guided diffusion using gradient queries on new samples (Algorithm \ref{alg:main}; Figure \ref{fig:framework} with fine-tuning block off). We give the first convergence theory showing generated samples converge to a regularized optimal solution w.r.t the objective (Theorem~\ref{thm:fully linear}) with linear score class \eqref{equ:fully linear class}. The regularization is imposed by the pre-trained diffusion model, revealing a fundamental limit of adapting pre-trained diffusion models with guidance. 

\vspace{-3pt}
\noindent $\bullet$ \textbf{Provide rate-matching optimization theory.} Furthermore, we propose an adaptive gradient-guided diffusion, where both pre-trained score network and guidance are iteratively updated using self-generated samples (Algorithm~\ref{alg:main_w_update}; Figure \ref{fig:framework} with the fine-tuning block turned on). We show in expectation its iteration converges to a global optima within the latent subspace, at a rate of $\tilde{\cO}(1/K)$ (Theorems~\ref{thm:w. update w. subspace}, $K$ is $\#$ of iterations), matching the classical convergence in convex optimization.

\vspace{-3pt}
\noindent $\bullet$ \textbf{Provide experimental justification.}  Simulation and image experiments are provided in \cref{sec:experiment} to support theoretical findings on latent structure-preserving and optimization convergence.

\section{Related Works}
To summarize the related work, we first give an overview of empirical studies relevant to our objective. We then discuss the theory of diffusion models, to which our main contribution is focused. Other related topics, such as direct latent optimization in diffusion models and a detailed review of sampling and statistical theory of diffusion models, are deferred to \cref{sec:r_w}.

\paragraph{Classifier Guidance and Training-free Guidance.} \citep{dhariwal2021diffusion} introduced classifier-based guidance, steering pre-trained diffusion models towards a particular class during inference. This method offers flexibility by avoiding task-specific fine-tuning, but still requires training a time-dependent classifier. Training-free guidance methods \cite{chung2022diffusion, song2023loss, yu2023freedom, bansal2023universal, he2023manifold, pan2023towards, garber2024image} eliminate the need for a time-dependent classifier, using only off-the-shelf loss guidance during inference. \cite{chung2022diffusion, song2023loss, he2023manifold, garber2024image} is a line of works solving inverse problems on image and \cite{yu2023freedom, pan2023towards} aims for guided/conditional image generation. Though not originally developed for solving optimization problems, \cite{chung2022diffusion, yu2023freedom} both propose a similar guidance to ours: taking gradient on the predicted clean data $x_0$ with respect to corrupted $x_t$. Differently,  our paper presents the first rigorous theoretical analysis of this gradient-based guidance approach. Furthermore, we propose an algorithm that iteratively applies the guidance as a module to the local linearization of the optimization objective, demonstrating provable convergence guarantees.

\paragraph{Fine-tuning of Diffusion Models.} Several methods for fine-tuning diffusion models to optimize downstream reward functions include RL-based fine-tuning \citep{black2023training,fan2023dpok} and direct backpropagation to rewards \citep{clark2023directly,prabhudesai2023aligning,xu2023imagereward,uehara2024fine}. However, these approaches often suffer from high computational costs and catastrophic forgetting in pre-trained models. Our guidance method is training-free and applied during the inference phase, eliminating the need to fine-tune diffusion models.

\paragraph{Theory of Diffusion Models.} Current theory works primarily focus on unconditional diffusion models. Several studies demonstrate that the distributions generated by diffusion models closely approximate the true data distribution, provided the score function is accurately estimated \citep{de2021diffusion, albergo2023stochastic, block2020generative, lee2022convergencea, chen2022sampling, lee2022convergenceb,chen2023restoration, chen2023probability, benton2023linear}. For conditional diffusion models, \citep{yuan2023reward, fu2024unveil} establish sample complexity bounds for learning generic conditional distributions. Our novel analysis establishes a connection between the sampling process in gradient-based guided diffusion and a proximal gradient step,  providing convergence guarantees.

\section{Preliminaries: Diffusion Models and Guidance} 
\label{sec:preliminaries}

Score-based diffusion models capture the distribution of pre-training data by learning a sequence of transformations to generate new samples from noise \citep{song2020score}.  A diffusion model comprises a forward and a backward process, for which we give a review as follows. 

\paragraph{Forward Process.} The forward process progressively adds noise to data, and then the sample trajectories are used to train the score function. The forward process initializes with $X_0 \in \mathbb{R}^{D}$, a random variable drawn from the pre-training data $\mathcal{D}$. It introduces noise via an Ornstein-Uhlenbeck process, i.e.,
\begin{equation}\label{eq:forward_sde}
\diff X_t = -\frac{1}{2} q(t) X_t \diff t + \sqrt{q(t)} \diff W_t ~~~ \text{for} ~~ q(t) > 0,
\end{equation}
where $(W_t)_{t\geq 0}$ is Wiener process, and $q(t)$ is  non-decreasing.  $X_t$ represents the noise-corrupted data distribution at time $t$. The conditional distribution $X_t | X_0 = x_0$ is Gaussian, i.e., $\cN (\alpha(t) x_0, h(t)I_D)$ with $\alpha(t) = \exp(-\int_0^t \frac{1}{2} q(s) ds)$ and $h(t) = 1 - \alpha^2(t)$. In practice, the forward process will terminate at a large time $T$ so that 
the marginal distribution of $X_T$ is close to $\cN(0, I_D)$.

\paragraph{Backward Process.}
If reversing the time of the forward process, we can reconstruct the original distribution of the data from pure noise. With $(\overline{W}_t)_{t\geq 0}$ being another independent Wiener process, the backward SDE below \citep{anderson1982reverse} reverses the time in the forward SDE \eqref{eq:forward_sde}, 
\begin{equation}\label{eq:backward}
\diff \bXb_t = \left[\frac{1}{2} \bXb_t + { \underbrace{\nabla \log p_{T-t}(\bXb_t)}_{\text{score}}} \right] \diff t + \diff \overline{W}_t.
\end{equation}
Here $p_t(\cdot)$ denotes the marginal density of $X_t$ in the forward process. In the forward SDE \eqref{eq:backward}, the \textit{score function} $\nabla \log p_t(\cdot)$ plays a crucial role, but it has to be estimated from data.

\vspace{-8pt}
\paragraph{Score Matching.} To learn the unknown score function $\nabla \log p_t(\cdot)$, we train a score network $s_{\theta}(x,t)$ using samples from forward process. 
Let $\mathcal{D}$ denote the data for training. Then the score network is learned by minimizing the following loss:
\begin{equation}\label{equ:score match objective}
\begin{aligned}
        \hbox{min}_{s\in \mathcal{S}} \int_{0}^{T} {\EE}_{x_0 \in \mathcal{D}}\EE_{x_t|x_0} \sbra{\dabs{\nabla_{x_t} \log \phi_t(x_t|x_0)-s(x_t, t)}^2} \dif t,
\end{aligned}
\end{equation}
where $\cS$ is a given function class,
${\EE}_\mathcal{D}$ denotes the empirical expectation over training data $\mathcal{D}$ and $\EE_{x_t|x_0}$ denotes condition expectation over the forward process, $\phi_t(x_t | x_0)$ is the Gaussian transition kernel, i.e., $(2\pi h(t))^{-D/2}\exp(-{\lVert x_t - \alpha(t)x_0 \rVert^2}/\rbra{2h(t)} )$.

\paragraph{Generation and Guided Generation.}
Given a pre-trained score function $s_{\theta}$, one generates samples by the backward process \eqref{eq:backward} with the true score replaced by $s_{\theta}$.
Further, one can add additional guidance to steer its output distribution towards specific properties, as formulated in Module \ref{mod:backwards}.

\begin{module}[h]
    \begin{algorithmic}[1]
    \small
    \caption{$\texttt{Guided\_BackwardSample}(s_{\theta}, \texttt{G})$}
    \label{mod:backwards}
    \label{alg:backward}
        \STATE {\bf Input}: Score $s_\theta$, guidance $\texttt{G}$ default to be zero for unguided generation. 
        \STATE {\bf Hyper-parameter}: $T$.
        \STATE  
        Initialized at $\bXb_t\sim\mathcal{N}(0,I)$, simulate the following SDE till time $T$:
        \begin{equation*}
            \diff \bXb_t = \left[\frac{1}{2} \bXb_t + { s_{\theta}\rbra{\bXb_t, T-t } +  \texttt{G}\rbra{\bXb_t, T-t } } \right]  \diff t + \diff \overline{W}_t.
        \end{equation*}
    \STATE {\bf Output}: {Sample $X_{T}^{\leftarrow}$. } 
    \end{algorithmic}
\end{module}

A common goal of guided generation (Module \ref{mod:backwards}) is to generate $X$ with a desired property $Y=y$ from the distribution $P(X|Y = y)$. To this end, it essentially needs to learn the \textbf{conditional score function} $\nabla_{x_t} \log p_t(x_t \mid y)$. The Bayes rule gives
\begin{equation}
\label{equ:bayes_rule} 
    \nabla_{x_t} \log p_t(x_t \mid y) = \underbrace{\nabla  \log p_t(x_t)}_{\text{est. by } s_{\theta}(x_t, t)} + \underbrace{\nabla_{x_t} \log p_t(y\mid x_t)}_{\text{to be est. by guidance}}.
\end{equation}
When a pre-trained score network $s_{\theta}(x_t, t)\approx \nabla \log p_t(x_t)$, what remains is to estimate $\nabla_{x_t} \log p_t(y\mid x_t)$ and add it as the guidance term $\texttt{G}$ to the backward process in Module \ref{mod:backwards}.

\paragraph{Classifier and Classifier-Free Guidance.} 
Classifier guidance \citep{song2020score, dhariwal2021diffusion} samples from $P(X|Y = y)$ when $Y$ is a discrete label. This method estimates $\nabla_{x_t} \log p_t(y\mid x_t)$ by training auxiliary classifiers, denoted as $\hat{p}(y \mid x_t, t)$, and then computing the gradient of the classifier logits as the guidance, i.e., $\texttt{G}(x_t, t) = \nabla_{x_t} \log \hat{p}(y \mid x_t, t)$. 
Alternatively, classifier-free guidance \citep{ho2022classifier} jointly trains a conditional and an unconditional diffusion model, combining their score estimates to generate samples.

\paragraph{Notations.} For a random variable $X$, $P_{x}$ represents its distribution, and $p(x)$ denotes its density. For $X$, $Y$ jointly distributed, $P(X \mid Y = y)$ denotes the conditional distribution, and $p(x \mid y)$ its density. The conditional expectation is denoted as $\EE[x \mid y] $. Let $\mathcal{D}$ be the pre-training data, and let ${\EE}_\mathcal{D}$ be the empirical expectation over $\mathcal{D}$. 
The empirical mean and covariance matrix are  $\bar \mu := {\EE}_{x \in \mathcal{D}} [x]$ and $\bar \Sigma := {\EE}_{x \in \mathcal{D}} [(x-\bar \mu) (x-\bar \mu)^\top]$.
For a matrix $A$, $\operatorname{Span}(A)$ denotes the subspace spanned by its column vectors, and for a square matrix $A$, $A^{-1}$ denotes its inverse or Moore–Penrose inverse. For any differentiable function $f : \mathbb{R}^n \to \mathbb{R}^m$, $\nabla f \in \mathbb{R}^{m \times n}$ denotes Jacobian matrix, i.e., $(\nabla f)_{ij} =\frac{\partial f_i(x)}{\partial x_j}.  $

\section{A Primer on Gradient Guidance}\label{sec:grad_guidance}
Let us start with stating the problem we want to study: suppose \textbf{given} a pre-trained diffusion model where the score network $s_{\theta}(x_t, t)$ provides a good approximation to the true score function $\nabla \log p(x_t)$, the \textbf{goal} is to generate novel samples with desired properties that can be measured by a user-specified differentiable function $f$. We will refer to $f$ as a reward or objective function later on. To achieve this goal, from \cref{sec:preliminaries}, we know that guided generation with diffusion model is a good candidate, which deploys the following guided backward process (Module \ref{mod:backwards}):
\begin{align*}
\diff \bXb_t = \left[\frac{1}{2} \bXb_t + s_{\theta}(\bXb_t,T - t) {\color{red}+ {\texttt{G}}(X_t^{\leftarrow}, t)} \right] \diff t + \diff \overline{W}_t.
\end{align*}
Here the guidance term ${\texttt{G}}$ is what we focus on and wish to design. Specifically, we want to construct this guidance term ${\texttt{G}}$ based on the gradient $\nabla f$ of a general objective $f$. This is motivated by the gradient methodology in optimization, a natural, intuitive way for adding guidance is to steer the generated samples towards the steepest ascent direction of $f$ \citep{chung2022diffusion,bansal2023universal, clark2023directly}.

\subsection{Structural Data Distribution with Subspace}
When incorporating property optimization in the generation process, it's crucial to consider intrinsic low-dimensional structures of real-world data, such as local regularities, global symmetries, and repetitive patterns \citep{tenenbaum2000global, roweis2000nonlinear,pope2021intrinsic}. 
Blindly improving $f$ at the cost of losing these structures degrades sample quality dramatically. This quality degradation, also known as ``reward over-optimization'', is a common challenge for adapting diffusion models towards an external reward \citep{yuan2023reward,uehara2024fine}. 

To study the design of guidance that mitigates the risk of over-optimization, we focus on data that admits a low-dimensional latent subspace, formulated in the following assumption. 

\begin{assumption}[Subspace Data]\label{asmp:subspace_data}
Data $X \in \RR^D$ can be represented as $X = A U$, where $A \in \RR^{D \times d}$ is an unknown matrix with orthonormal columns, and the latent variable $U \in \RR^d$ follows some distribution $P_u$ with a density $p_u$. Here $d \ll D$. The empirical covariance of $U$ is assumed full rank.
\end{assumption}

In the rest of this section, we investigate the principles for designing a guidance based on the gradient of $f$ that ensures (i) improving the value of $f$, and at the same time, (ii) being adhere to the subspace structure, i.e. generated samples being close to the subspace spanned by $A$.

\subsection{Naive Gradient Does't Work as Guidance}
A tempting simple choice of the guidance ${\texttt{G}}$ is by taking the steepest ascent direction $\nabla f$, which we refer to as {\it naive gradient guidance} i.e.,
\begin{equation}
\label{equ:grad_guidanace_conceptual}
{\texttt{G}}(X_t^{\leftarrow}, t) \propto \nabla f(X_t^{\leftarrow}).
\end{equation}

However, the naive gradient guidance \eqref{equ:grad_guidanace_conceptual} would jeopardize the latent structure of data, which is demonstrated by the following proposition:

\begin{proposition}[Failure of Naive Guidance]\label{prop:failure_naive}
For naive guidance $\texttt{G}(X_t^{\leftarrow}, t) = b(t) \nabla f(X_t^{\leftarrow})$, suppose $b(t) >b_0>0$ for $t >t_0.$ For data in subspace under Assumption~\ref{asmp:subspace_data} and reward $f(x)=g^\top x$, $g {\perp} \operatorname{Span}(A)$ with $h(t)=1-\exp(-\sqrt{t})$, then the off-subspace component of the generated sample is consistently large:
$$
    \mathbb{E} [X_{T,\perp}^{\leftarrow}] = C g, \quad C > \exp{\left(-5/2\right)}b_0.
$$
\end{proposition}

The intuition provided by \cref{prop:failure_naive} is, while the pre-trained score network effectively steers the distribution toward the latent subspace \citep{chen2023score}, the gradient vector  $\nabla f$ may point outside the subspace, causing the generated output to deviate from it (Figure~\ref{fig:out_of_subspace}). This is why naive gradient guidance fails. \cite{uehara2024fine} also observed this, explaining that  $\nabla f$ is not computed for $t=T$ i.e., it is not aligned with the clean data space.

\begin{figure}[!htb]
\centering
\vspace{-15pt}
\subfigure{\includegraphics[width=0.35\textwidth]{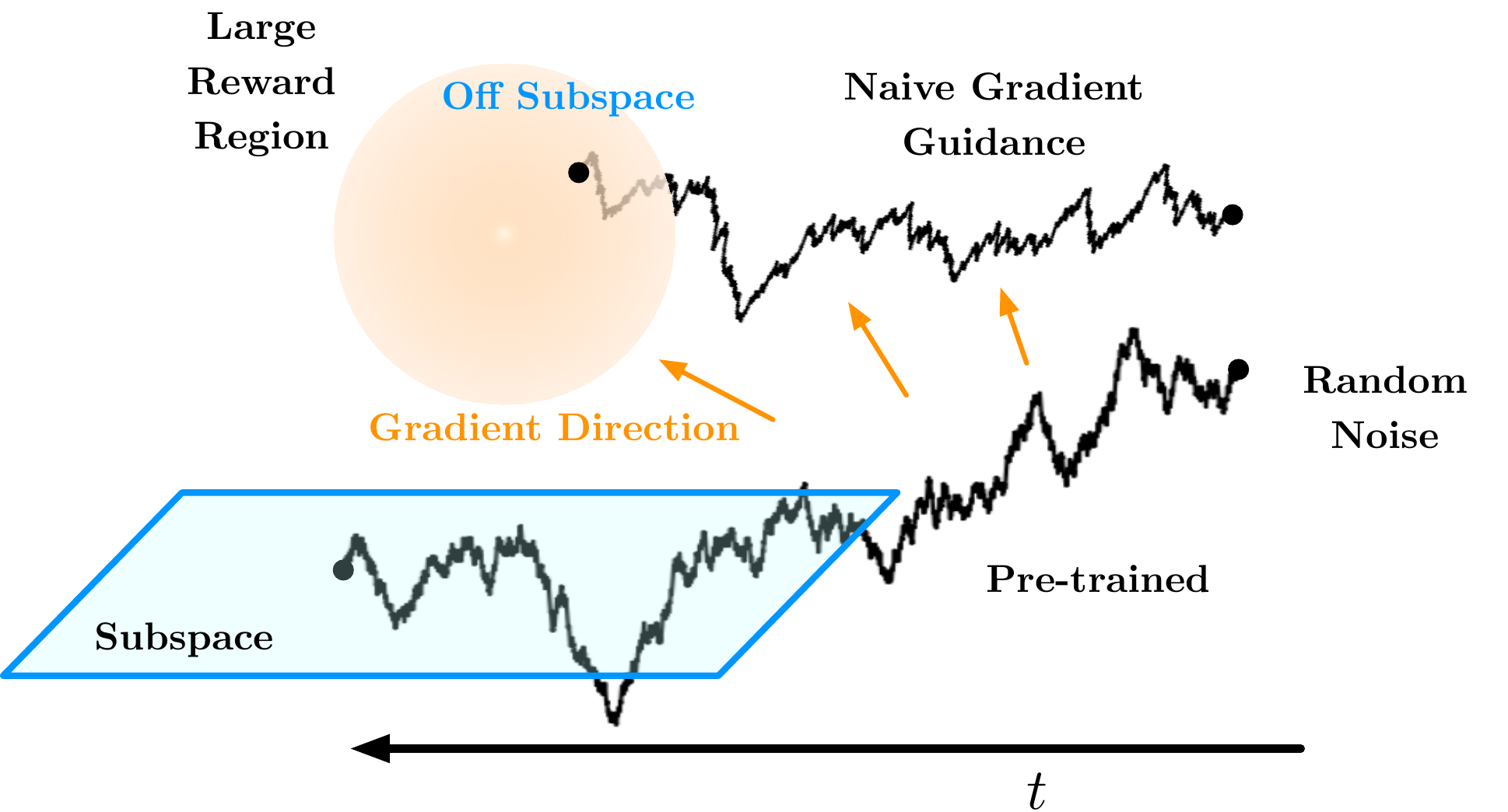}}\qquad
\subfigure{\includegraphics[width=0.33\textwidth]{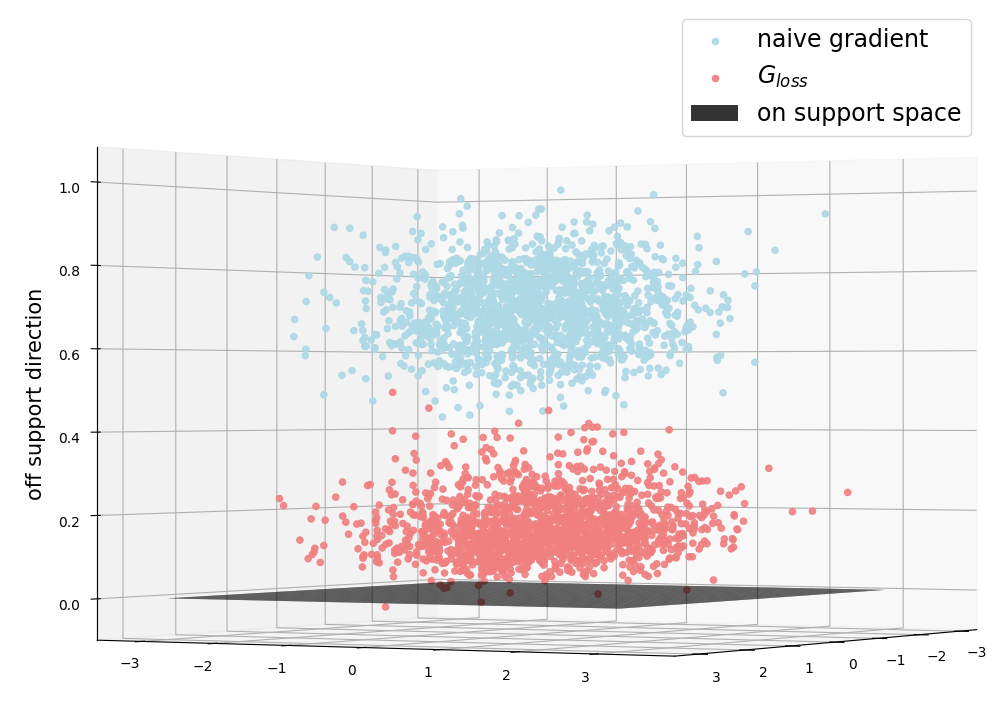} }
\vspace{-8pt}
\caption
{\small{
\textbf{Directly adding the gradient of the objective function to the backward process sabotages the subspace structure.} 
Left: Directly adding gradients that point out of the data subspace causes samples to leave the subspace.
Right: Numerical experiments show that naive gradients lead to substantially larger off-subspace error compared to our gradient guidance $\texttt{G}_{loss}$(Definition~\ref{def:gradient_guidance+}); see \cref{sec:experiment} for experiment details.}
}
\label{fig:out_of_subspace}
\end{figure}

\subsection{Motivating Latent Subspace Preserving Guidance from Conditional Score Function}
\label{sec:connection}
Failure of the naive gradient in maintaining data structure motivates us to seek alternatives. To get some inspiration, we start with the most elementary Gaussian probabilistic model and linear $f$. Later we will drop this assumption and consider general data distributions and general $f$.
\begin{assumption}[Gaussian Linear model]\label{asump: gaussain data linear reward}
Let data follow a Gaussian distribution, i.e., $X\sim \cN(\mu, \Sigma)$, and let $f(x) = g^\top x$ be a linear function for some $g \in \RR^D$. Let $Y =  f(X) + \epsilon$  with independent, identically distributed noise $\epsilon \sim \cN(0, \sigma^2)$ for some $\sigma > 0$. 
\end{assumption}
By the Bayes' rule, the conditional score $\nabla_{x_t} \log p_t(x_t \mid y)$ takes the form of a sum given by
\begin{equation}
\tag{recall \eqref{equ:bayes_rule}} 
    \nabla_{x_t} \log p_t(x_t \mid y) = \underbrace{\nabla  \log p_t(x_t)}_{\text{est. by } s_{\theta}(x_t, t)} + \underbrace{\nabla_{x_t} \log p_t(y\mid x_t)}_{\text{to be est. by guidance}}.
\end{equation}
Under the Gaussian assumption, we derive the following closed-form formula of the guidance term $\log p_t(y\mid x_t)$ that we want to estimate. The proof is provided in Appendix~\ref{prf:lmm:guidance=cond w.o.subspace}. 
\begin{lemma}[Conditional Score gives a Gradient-like Guidance]\label{lmm:guidance=cond w.o.subspace}
Under Assumption~\ref{asump: gaussain data linear reward}, we have
\begin{equation}
\label{equ:guidance_gaussian}
      \nabla_{x_t} \log p_t(y \mid x_t) = - \rbra{2 \sigma^2_y(x_t)}^{-1} \cdot \nabla_{x_t} \rbra{y - g^{\top} \EE[x_0 \mid x_t] }^2,
\end{equation}
 where $\EE[x_0 | x_t] $ denotes the conditional expectation of $x_0$ given $x_t$ in the forward process \eqref{eq:forward_sde},
 and  $\sigma_y^2(x_t)$ is the variance of the conditional distribution $Y \mid X_t =x_t$.
\end{lemma}
The form of conditional score shown in \cref{lmm:guidance=cond w.o.subspace} motivates our proposed gradient guidance:
\begin{definition}[Gradient Guidance of Look-Ahead Loss] \label{def:gradient_guidance+}
Given a gradient vector $g$, define the \textit{gradient guidance of look-ahead loss} as
\begin{equation}\label{equ:gradient_guidance+}
        \texttt{G}_{loss}(x_t, t) := - \beta(t) \cdot \nabla_{x_t}\rbra{y - g^\top \EE[x_0|x_t]}^2,
\end{equation}
where $\beta(t) > 0, y \in \RR$ are tuning parameters, and $\EE[x_0|x_t]$ is the conditional expectation of $x_0$ given $x_t$ in the forward process \eqref{eq:forward_sde}, i.e., $\diff X_t = -\frac{1}{2} q(t) X_t \diff t + \sqrt{q(t)} \diff W_t.$
\end{definition}
The formula in \eqref{equ:gradient_guidance+} generalizes the intuition of a conditional score for any data distribution and objective function. It scales with the residual term $y - g^{\top} \EE[x_0 \mid x_t]$, tuning the {\it strength of guidance}. Here,  $\EE[x_0 \mid x_t]$ represents the expected clean data $x_0$ given $x_t$ in the forward process, which coincides with the expected sample in the backward view. This residual measures the \textbf{look-ahead gap} between the expected reward of generated samples and the target value. The \textbf{look-ahead loss} $(y - g^\top \EE[x_0|x_t])^2$  resembles the proximal term commonly used in first-order proximal optimization methods. 
\vspace{-4pt}
\paragraph{Remark.} The gradient guidance \eqref{equ:gradient_guidance+} aligns with the groundtruth conditional score in \eqref{equ:guidance_gaussian} under the assumptions of Gaussian data and linear reward (\cref{asump: gaussain data linear reward}). This theoretical motivation, rooted in a fundamental framework, distinguishes our work from the empirical practice, such as DPS \citep{chung2022diffusion} and universal guidance \citep{bansal2023universal}.

A key advantage of $\texttt{G}_{loss}$ is that it enables preserving the subspace structure, for \textbf{any} data distribution under \cref{asmp:subspace_data}. This is formalized in the following theorem, the full proof in \cref{prof:faith_grad_guidance}.

\begin{theorem}[Faithfulness of $\texttt{G}_{loss}$ to the Low-Dimensional Subspace of Data]
\label{thm:faith_grad_guidance}
Under Assumption~\ref{asmp:subspace_data}, it holds for any data distribution and $g\in\mathbb{R}^D$ that
     \begin{equation}
     \label{equ:guidance_on_subspace}
         \texttt{G}_{loss}(x_t, t) \in \operatorname{Span}(A).
     \end{equation}  
\end{theorem}
\vspace{-12pt}
\paragraph{Proof Sketch.}
We have
\begin{equation*}
    \nabla_{x_t}\rbra{y - g^\top \EE[x_0 \mid x_t]}^2  \propto \nabla_{x_t} \EE[x_0 \mid x_t]^{\top} g.
\end{equation*}
We will show that the Jacobian $\nabla_{x_t} \EE[x_0 | x_t]$ maps any vector $g \in \mathbb{R}^D$ to $\operatorname{Span}(A)$. To see this, we utilize the score decomposition result in \cref{sec:score_decomposition}
and plug it  into the equality ${\EE}[x_0|x_t] = \alpha^{-1}(t)\rbra{x_t + h(t)\nabla \log p_t(x_t)}$ (Tweedie's formula \citep{efron2011tweedie}),
we have
\begin{align}
\label{eqn:Ex0_xt}
{\EE}[x_0 \mid x_t] = {\alpha^{-1}(t)} \left(x_t + h(t) \left[Am(A^\top x_t) - h^{-1}(t) x_t\right]\right) = {h(t)}/{\alpha(t)} \cdot A m(A^\top x_t),
\end{align}
here $m(u) = \nabla \log p_t^{\sf LD}(u) + {h^{-1}(t)}u$, $p_t^{\sf LD}(u)$ latent density (\cref{sec:score_decomposition}). We see $\nabla_{x_t} \EE[x_0 | x_t]^{\top}$ maps any vector to $\operatorname{Span}(A)$ because $m(\cdot)$ takes $A^\top x_t$ as input in the expression of ${\EE}[x_0 | x_t]$.
$\hfill \blacksquare$

We highlight that the faithfulness of  $\texttt{G}_{loss}$ holds for {\it  arbitrary} data distribution supported on the latent subspace. 
It takes advantage of the score function's decomposition \eqref{equ:decomposition}, having the effect of automatically adapting $g$ onto the latent low-dimensional subspace of data.
\vspace{-4pt}
\paragraph{Remark.} We provide a rigorous guarantee for manifold preservation of gradient guidance, a property previously discussed by \citep{chung2022diffusion,chung2022improving,he2023manifold}. However, while \citep{chung2022diffusion,chung2022improving} claim manifold preservation, they do not present a formal mathematical proof.  \citep{he2023manifold} relies heavily on pre-trained autoencoders for manifold projections, which are often unavailable in practical scenarios.

\subsection{Estimation and Implementation of $\texttt{G}_{loss}$}\label{sec:Gloss_implementation}

In this section, we discuss the estimation and computation of $\texttt{G}_{loss}$ based on a pre-trained score function $s_\theta$ in practice. $\texttt{G}_{loss}$ involves the unknown quantity $\EE[x_0|x_t]$. One can construct estimate $\EE[x_0 | x_t]$  by considering the Tweedie's formula \citep{efron2011tweedie}:
$     \nabla \log p_t(x_t) = - h^{-1}(t)\EE\left[x_t - \alpha(t)x_0 \big| x_t\right], $
which gives rise to
\begin{equation}
\label{equ:hatx0_guidance}
        \hat{\EE}[x_0|x_t]  :=  \alpha^{-1}(t)\rbra{x_t + h(t){s}_{\theta}(x_t, t)},
\end{equation}
and we refer to it as the \textit{look-ahead estimator}. The estimator \eqref{equ:hatx0_guidance} is widely adopted in practice \citep{song2020denoising, bansal2023universal}. Here $\alpha(t)$ and $h(t)$ are the noise scheduling used in the forward process \eqref{eq:forward_sde}.

Thus, we have obtained an \textbf{implementable version of the gradient guidance $\texttt{{G}}_{loss}$}, given by
\begin{equation}\label{equ:gradient_guidance_est}
        \texttt{{G}}_{loss}(x_t, t) = - \beta(t) \cdot \nabla_{x_t}\sbra{y - g^\top \left( \alpha^{-1}(t)\rbra{x_t + h(t){s}_{\theta}(x_t, t)}\right)}^2,
\end{equation}
\begin{wrapfigure}{r}{0.32\textwidth}
\centering
\vspace{-5pt}
\includegraphics[width=0.38\textwidth]{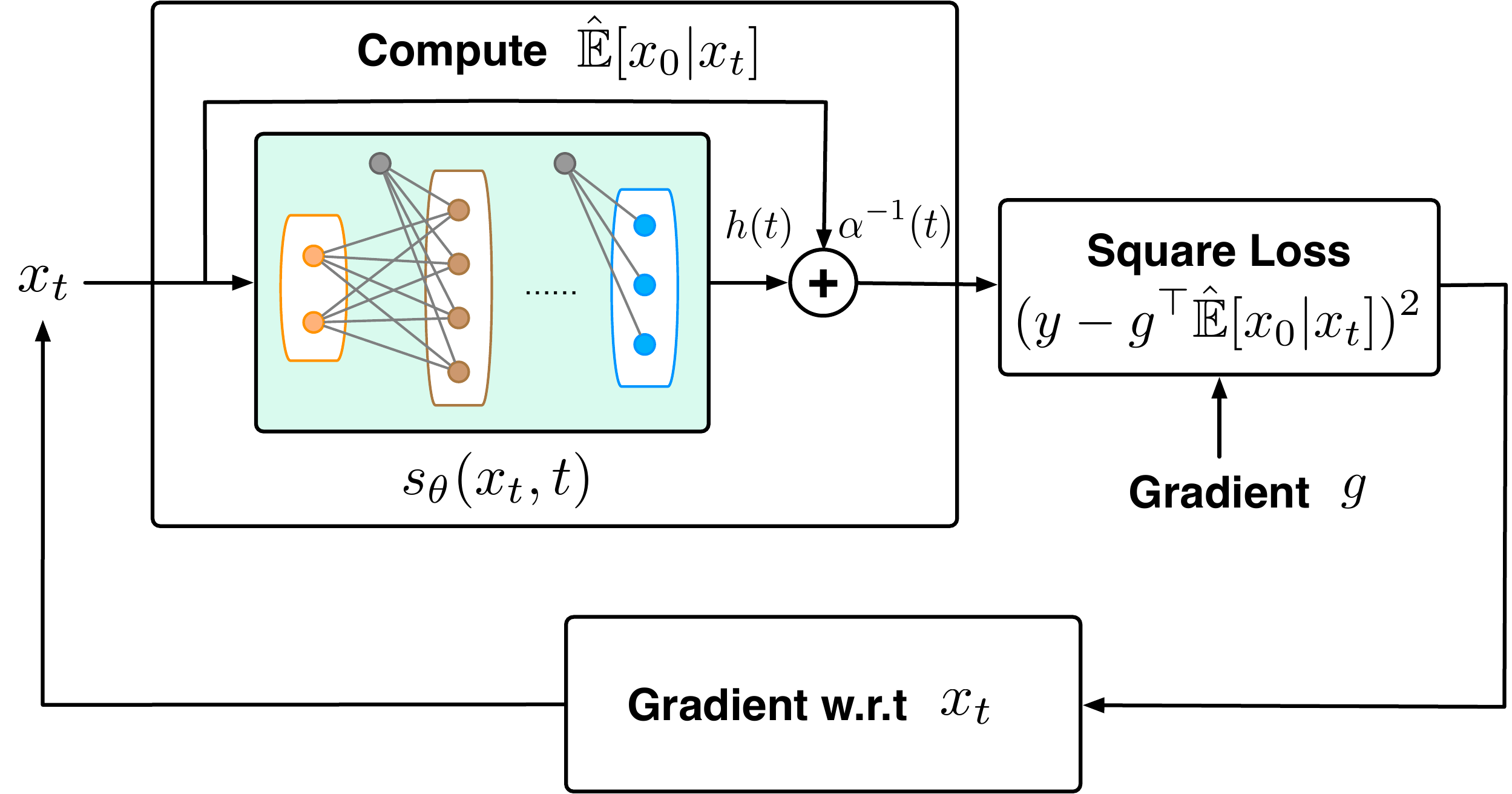}
\vspace{-18pt}
\caption{Computing $\texttt{G}_{loss}$.}
\vspace{-17pt}
\label{fig:auto_grad}
\end{wrapfigure}
With a slight abuse of notation, we use $\texttt{{G}}_{loss}$ to refer to this implementable formula \eqref{equ:gradient_guidance_est} in the remainder of this paper.
Here, $y$ is a target reward value from conditional score analysis under a Gaussian model and is treated as a tuning parameter in practice.
The gradient guidance \eqref{equ:gradient_guidance_est} is lightweight to implement. Given a pre-trained score function $s_{\theta}$ in the form of a neural network,  computing \eqref{equ:gradient_guidance_est} involves calculating the squared loss $\rbra{y - g^\top \hat{\EE}[x_0|x_t]}^2$ via a forward pass of $s_{\theta}$ and a backward pass using the auto-gradient feature of deep-leaning frameworks such as PyTorch and TensorFlow. 
See Figure~\ref{fig:auto_grad} for illustration.

\section{Gradient-Guided Diffusion Model as Regularized Optimizer}\label{sec:reg opt}

In this section, we study if gradient guidance steers pre-trained diffusion models to generate \textbf{near-optimal} samples. Our results show that: 1) Iterative gradient guidance improves the objective values; 2) The pre-trained diffusion model acts as a regularizer from an optimization perspective.

\subsection{Gradient-Guided Generation with A Pre-trained Score}

Assuming access to a pre-trained score network $s_\theta$ and the gradient of the objective function $f$, we present Algorithm~\ref{alg:main} to adapt the diffusion model and iteratively update the gradient guidance \eqref{equ:gradient_guidance_est}.
See Figure \ref{fig:framework} for illustration.

Alg.~\ref{alg:main} takes any pre-trained score function $s_{\theta}$ as input. Each iteration evaluates $\nabla f(\cdot)$ at samples from the previous iteration (Line 5(i)), computes the gradient guidance $\texttt{G}_{loss}$ with the new gradient (Line 5(ii)), and generates new samples using the updated guidance (Module~\ref{mod:backwards}). The algorithm outputs an adapted diffusion model specified by $(s_{\theta}, \texttt{G}_{K})$.
\vspace{-5pt}
\begin{algorithm}[h]
\begin{algorithmic}[1]
\small
\caption{Gradient-Guided Diffusion for Generative Optimization}
\label{alg:main}
\STATE {\bf Input}: Pre-trained score network $s_{\theta}(\cdot,\cdot)$, differentiable objective function $f$.
\STATE {\bf Tuning Parameter}: 
Strength parameters $\beta(t)$, $\{y_k\}_{k=0}^{K-1}$, number of iterations $K$, batch sizes $\{ B_k \}$.
\STATE {\bf Initialization}: $\texttt{G}_{0} = \texttt{NULL}$.
\FOR{$k = 0 ,\dots, K-1$}
\STATE {\bf Generate}: Sample 
$z_{k,i} \sim$ $ \texttt{Guided\_BackwardSample} (s_{\theta}, \texttt{G}_{k})$ using Module~\ref{mod:backwards}, for $i \in [B_{k}]$. 
\STATE{\bf Compute Guidance}: \\
        (i) Compute the sample mean $\Bar{z}_{k} := (1/B_k)\sum_{i=1}^{B_k} z_{k,i} $.\\ (ii) Query gradient $g_k= \nabla f(\Bar{z}_{k})$.\\
        (iii) Update gradient guidance $\texttt{G}_{k+1}(\cdot,\cdot) = \texttt{G}_{loss}(\cdot,\cdot) $ via \eqref{equ:gradient_guidance+}, using $s_{\theta}$, gradient vector $g_k$, and parameters $y_k$ and $\beta(t)$. 
\ENDFOR
\STATE {\bf Output}: { ($s_{\theta} , \texttt{G}_{K} $)}.
\end{algorithmic}
\end{algorithm}
\vspace{-5pt}

\subsection{Gradient-Guided Diffusion Converges to Regularized Optima in Latent Space}\label{sec:theorems}

We analyze the convergence of Alg.~\ref{alg:main} and show that in final iterations,  generated samples center around a regularized solution of the optimization objective $f$ within the subspace $\operatorname{Span}(A)$. Our theorems allow the pre-training data to have {\it arbitrary distribution}.

\begin{assumption}[Concave smooth objective]
\label{asmp:f}
The objective $f: \RR^D \to \RR$ is concave and $L$-smooth w.r.t. the (semi-)norm $\dabs{\cdot}_{\bar{\Sigma}^{-1}}$, i.e., $\dabs{\nabla f(x_1) - \nabla f(x_2)}_{\bar{\Sigma}} \leq L \dabs{x_1-x_2}_{\bar{\Sigma}^{-1}}$ for any $x_1,x_2$.
\end{assumption}

While Alg.~\ref{alg:main} works with any pre-trained score network, we study its optimization properties focusing on the class of linear score functions given by 
\begin{equation}\label{equ:fully linear class}
\cS = \cbra{s(x, t)=  C_t x + b_t : C_t \in \RR^{D\times D}, \, b_t \in \RR^D}.
\end{equation}  
\vspace{-18pt}
\paragraph{Remark on the linear parametrization of score network \eqref{equ:fully linear class}:} 
Analyzing the output distribution of guided diffusion is challenging because the additional guidance term destroys the dynamics of reverse SDE. A linear score is a natural and reasonable choice for characterizing the output distribution, as it was also adopted by \citep{marion2024implicit}.

With a linear score function \eqref{equ:fully linear class}, pre-training a diffusion model is equivalent to using a Gaussian model to estimate and sample from the estimated distribution. Thus, the guidance $\texttt{G}_{loss}$ is also linear in $x_t$,  and the final output follows a Gaussian distribution; see \eqref{equ:generated_gaussian} in \cref{appd:prof thms}. We focus on the mean, $\mu_K$, of the generated distribution from the backward sampling of $(s_{\theta} ,\texttt{G}_K)$ (as $T\to\infty$), and establish its optimization guarantee.

\begin{theorem}[Convergence to Regularized Maxima in Latent Subspace in Mean]\label{thm:subspace fully linear} 
Let Assumptions \ref{asmp:subspace_data} and \ref{asmp:f} hold. Suppose we use the score function class \eqref{equ:fully linear class} for pre-training and computing guidance. Then Alg.\ref{alg:main} gives an adapted diffusion model that generates new samples that belong to  $\operatorname{Span}(A) $.  
Further, for any $\lambda > L,$ there exists $\beta(t), \cbra{y_k}$ and batch size $B_k $, such that with high probability  $1 - \delta$, the mean of the output distribution $ \mu_K$ converges to be near $x^*_{A, \lambda}$, and it holds
\begin{equation*} 
    f\rbra{x^*_{A, \lambda}} - f(\mu_K)   = \lambda \rbra{ \fs{L}{\lambda}}^K \mathcal{O}\rbra{d\log \rbra{\fs{K}{\delta}}}, 
\end{equation*}
where $x^*_{A, \lambda}$ is  an optimal solution of the regularized objective:
\begin{equation}\label{equ:reg_obj w. subspace fully linear}
        x^*_{A, \lambda} = \argmax_{x\in\operatorname{Span}(A)}~ \left\{ f(x) - \fs{\lambda}{2}\dabs{x-\bar{\mu}}^2_{{\Bar{\Sigma}}^{-1}} \right\}.
\end{equation}
where $\bar \mu,\bar\Sigma$ are empirical mean and covariance of pre-training data $\mathcal{D}.$
\end{theorem}

\vspace{-12pt}
\paragraph{Remarks.} \textbf{(1)} The regularization term $\fs{\lambda}{2}\dabs{x-\bar{\mu}}^2_{\Bar{\Sigma}^{-1}}$  \eqref{equ:reg_obj w. covariance} centers the data's mean $\bar\mu$ and is stronger in directions where the original data has low variance. Thus, the pre-trained score acts as a "prior" in the guided generation, favoring samples near the pre-training data, even with guidance. \\
\textbf{(2)} The regularization term cannot be arbitrarily small, as our theorem requires $\lambda \geq L$.  Thus, only adding gradient guidance cannot achieve the global maxima. If the goal is global optima, the pre-trained score must be updated and refined with new data, as explored in \cref{sec:update pre-score}.
\\
\textbf{(3)} The convergence rate is linear in the latent dimension $d$, rather than data dimension $D$. Since $\texttt{G}_{loss}$ is faithful to the latent subspace (\cref{thm:subspace fully linear}), the generated samples and optimization iterates of Alg.~\ref{alg:main} remain within $\operatorname{Span}(A)$. This leverage of the latent structure results in faster convergence.

\section{Gradient-Guided Diffusion with Adaptive Fine-Tuning for Global Optimization}\label{sec:update pre-score}

In the previous section, we have seen that adding guidance to a pre-trained diffusion model can't improve the objective function unlimitedly due to the pre-trained score function acting as a regularizer. We consider adaptively fine-tuning pre-trained diffusion models to attain global optima. Empirically, fine-tuning diffusion models using self-generated samples has been explored by \cite{black2023training, clark2023directly}.

\subsection{Adaptive Fine-Tuning Algorithm with Gradient Guidance}

We propose an adaptive version of the gradient-guided diffusion, where both the guidance and the score are iteratively updated using self-generated samples. The full algorithm is given in \cref{alg:main_w_update}.

We introduce a weighting scheme to fine-tune the score network using a mixture of pre-training data and newly generated samples. 
In Round $k$, let $\mathcal{D}_1, \dots, \mathcal{D}_k$ be sample batches generated from the previous rounds. 
Let $\cbra{w_{k,i}}_{i=0}^{k}$ be a set of weights. Conceptually, at Round $k$, we update the model by minimizing the weighted score matching loss:
\begin{equation}\label{equ:score update objective}
\begin{aligned}
    \min_{s \in \mathcal{S}} \int_{0}^{T} \sum_{i=0}^{k}w_{k,i}{\EE}_{x_0 \in \cD_i} \EE_{x_t|x_0} \sbra{\dabs{\nabla_{x_t} \log \phi_t(x_t|x_0)-s(x_t, t)}_2^2} \dif t,
\end{aligned}
\end{equation}
where $\cD_0 := \cD$ is the pre-training data. 
For illustration, please see also Figure \ref{fig:framework}, and the practical implementation  of Alg. \ref{alg:main_w_update} is in \cref{sec:experiment details}.

\begin{algorithm}[h]
\begin{algorithmic}[1]
\small
\caption{Gradient-Guided Diffusion with \textcolor{blue}{\textbf{Adaptive Fine-tuning}}}
\label{alg:main_w_update}
\STATE {\bf Input}: Pre-trained score $s_{\theta}(\cdot,\cdot)$,  differentiable objective function $f$. 
\STATE {\bf  Tuning Parameter}: 
strength parameter $\beta(t)$, $\{y_k\}_{k=0}^{K-1}$,  \textcolor{blue}{weights $\{\{w_{k,i}\}_{i=0}^k\}_{k=0}^{K-1}$}, number of iterations $K$,  batch sizes $\{ B_k \}$.\\

\STATE {\bf Initialize}: $s_{\theta_0} = s_{\theta}$, $\texttt{G}_0 = \texttt{NULL} $.

\FOR{$k = 0 ,\cdots, K-1$}
    \STATE {\bf Generate}: Sample a batch $\cD_{k} = \cbra{z_{k,i}}_{i=1}^{B_k}$ from $\texttt{Guided\_BackwardSample} (s_{\theta_{k}}, \texttt{G}_{k})$ (Module~\ref{mod:backwards}).
    \STATE{\bf Compute Guidance:}\\
        (i) Compute sample mean $\bar{z}_k = (1/B_k)\sum_{i=1}^{B_k} z_{k,i} $, and query gradient $g_k=  \nabla f(\bar{z}_k)$. \\
        (ii) \textcolor{blue}{Update $s_{\theta_{k}}$ to $s_{\theta_{k+1}}$ by minimizing the re-weighted objective \eqref{equ:score update objective}.}\\
        (iii) Compute $\texttt{G}_{k+1}(\cdot,\cdot)=\texttt{G}_{loss}(\cdot,\cdot)$ in \eqref{equ:gradient_guidance+}, using $s_{\theta_{k+1}}$ and $g_k$, with parameter $y_k, \beta(t)$.

\ENDFOR
\STATE {\bf Output:} {$(s_{\theta_K},\texttt{G}_{K}) $}.  
\end{algorithmic}
\end{algorithm}

\subsection{Guided Generation Finds Unregularized Global Optima}

Finally, we analyze the optimization properties for gradient-guided diffusion model with iterative finetuning. We establish that the process of  \cref{alg:main_w_update} yields a final output distribution whose mean, denoted by $\mu_K$, converges to the global optimum of $f$.

For simplicity of analysis, we study the following function class
\begin{equation}
\label{equ:fully linear class_freeze covariance}
\cS^{\prime} = \cbra{s(x, t)=  \hat{C_t} x + b_t : \, b_t \in \RR^D},
\end{equation}
where $\hat{C_t}$ is set to stay the same in pre-trained scores, with only $b_t$ updated during iterative fine-tuning.

\begin{theorem}[Convergence to Unregularized Maxima in Latent Subspace in Mean]\label{thm:w. update w. subspace}
Let Assumptions \ref{asmp:subspace_data} and \ref{asmp:f} hold, and assume there exists $M > 0$ such that $\dabs{x^*_{A, \lambda}} < M$ for all $\lambda \geq 0$. Suppose we use the score function class \eqref{equ:fully linear class} for pre-training $s_{\theta}$ and the class \eqref{equ:fully linear class_freeze covariance} for finetuning it. Then Algorithm \ref{alg:main_w_update} gives an adapted diffusion model that generates new samples belonging to  $\operatorname{Span}(A) $. Further, there exists $\{\beta(t)\}, \cbra{y_k}, \cbra{B_k}$ and $\cbra{w_{k,i}}$, such that  
with  probability $1-\delta$, 
\begin{equation}
    f^*_{A} -  f(\mu_K) = \cO \rbra{\fs{dL^2\log K}{K} \cdot \log \rbra{\frac{K}{\delta}}},
\end{equation}
where $
f^*_{A} = \max\{ f(x) | x\in \operatorname{Span}(A)\}.
$
\end{theorem}
Theorem \ref{thm:w. update w. subspace} illustrates that fine-tuning a diffusion model with self-generated data can reach global optima while preserving the latent subspace structure. 
The convergence rate matches standard convex optimization in terms of gradient evaluations, $K$.
Compared to standard gradient solvers, guided diffusion models leverage pre-training data to solve optimization problems in a low-dimensional space, preserving desired structures and enabling more efficient exploration and faster convergence.

\section{Experiments}\label{sec:experiment}

\subsection{Simulation}
We conduct numerical simulations of Algorithms \ref{alg:main} and \ref{alg:main_w_update}, following the subspace setup described in Assumption \ref{asmp:subspace_data}. Specifically, we set \( d = 16 \), \( D = 64 \),  The latent variable \( u \) is drawn from \( \cN(0, I_d) \) and used to construct \( x = Au \), where \( A \) is a randomly generated orthonormal matrix. We define the objective function \( f(x) = 10 - (\theta^\top x - 3)^2 \). To approximate the score function, we employ a version of the U-Net \citep{ronneberger2015u} with 14.8M trainable parameters. More details including how to set up $\theta$ are provided in \cref{sec:simulation_detail}.

\begin{wrapfigure}{r}{0.46\textwidth}
\centering
\vspace{-15pt}
\subfigure{\includegraphics[width=0.22\textwidth]{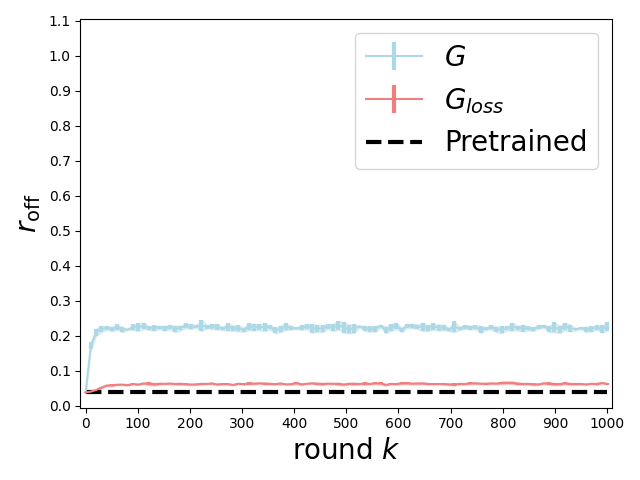}}
\subfigure{\includegraphics[width=0.22\textwidth]{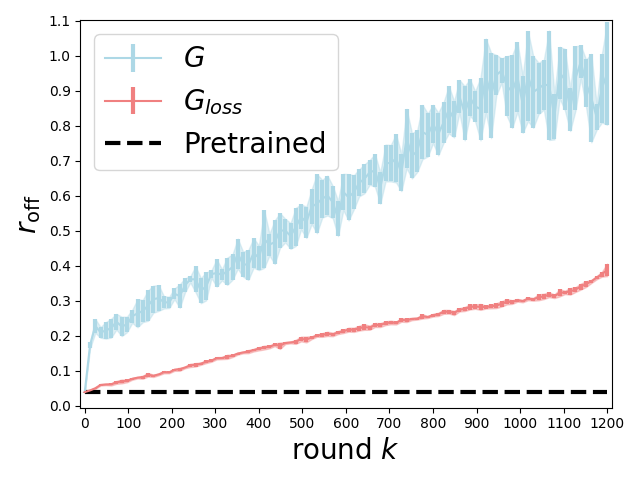}}
\vspace{-12pt}
\caption{\small{
\textbf{Comparison between two types of gradient guidance $\texttt{G}$ and $\texttt{G}_{loss}$ (left: Alg.~\ref{alg:main}; right: Alg.~\ref{alg:main_w_update}).} The off/on support ratio of the generated samples is defined as $r_\mathrm{off} =\frac{\Vert{x_\bot}\Vert}{\Vert{x_\Vert}\Vert}$.
}}
\vspace{-35pt}
\label{fig:compare_two_G}
\end{wrapfigure}
\vspace{-8pt}
\paragraph{Preserving Subspace Structure.} We first demonstrate that $\texttt{G}_{loss}$ preserves the subspace structure learned from the pre-trained model. For comparison, we also tested the naive guidance $\texttt{G}(x_t,t) :=  \beta(t) \left( y - g^\top \mathbb{E}[x_0 | x_t] \right) g$ (more details in \cref{sec:simulation_detail}.). \cref{fig:compare_two_G} (left) shows that $\texttt{G}_{loss}$ performs much better than the naive gradient $\texttt{G}$ in preserving the linear subspace. \cref{fig:compare_two_G} (right) demonstrates that off-support errors increase with adaptive score fine-tuning (Alg.~\ref{alg:main_w_update}) due to distribution shift, with $\texttt{G}$ resulting in more severe errors than $\texttt{G}_{loss}$.

\vspace{-5pt}
\paragraph{Convergence Results.} \cref{fig:convergence_simulation} (a) and (b) show that Alg.~\ref{alg:main} converges to a sub-optimal objective value, leaving a gap to the maximal value. This aligns with our theory that the pre-trained model acts as a regularizer in addition to the objective function. \cref{fig:convergence_simulation} (c) shows that Alg~\ref{alg:main_w_update} converges to the maximal value of the objective function. As illustrated by \cref{fig:convergence_simulation} (d), samples from Alg. \ref{alg:main} mostly stay close to the pre-training data distribution (dotted contour area), whereas samples from Alg. \ref{alg:main_w_update} move outside the contour as the diffusion model is fine-tuned with self-generated data.

\begin{figure}[!htb] 
\centering
\subfigure[ Alg.~\ref{alg:main}: $\theta=A\beta^*$]{\includegraphics[width=0.23\textwidth]{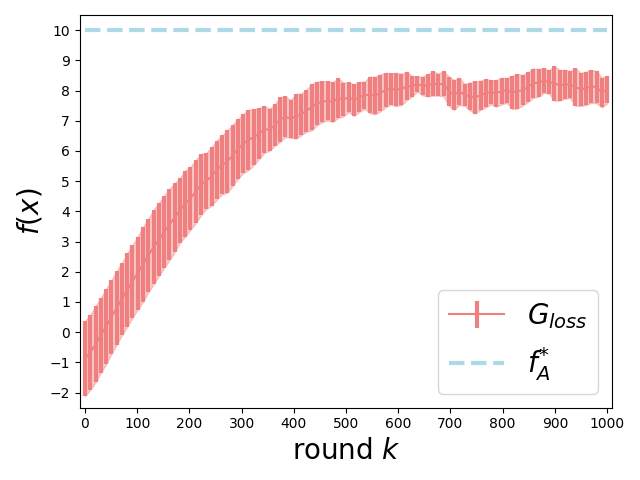}}
\subfigure[ Alg.~\ref{alg:main}: $\frac{\Vert{\theta_\bot}\Vert}{\Vert{\theta_\Vert}\Vert}=9$]{\includegraphics[width=0.23\textwidth]{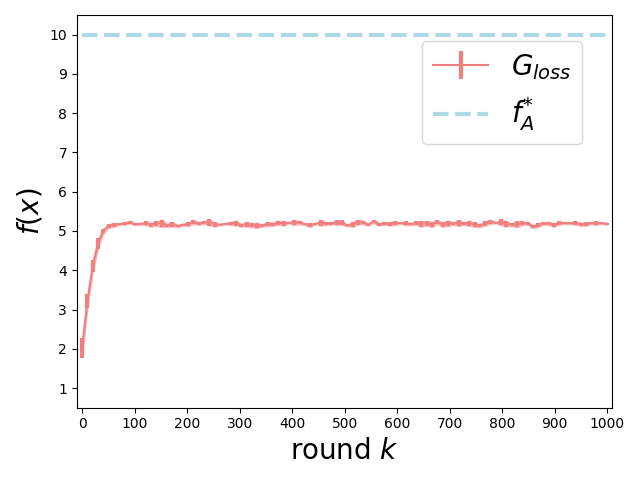}}
\subfigure[Convergence of Alg.~\ref{alg:main_w_update}]{\includegraphics[width=0.23\textwidth]{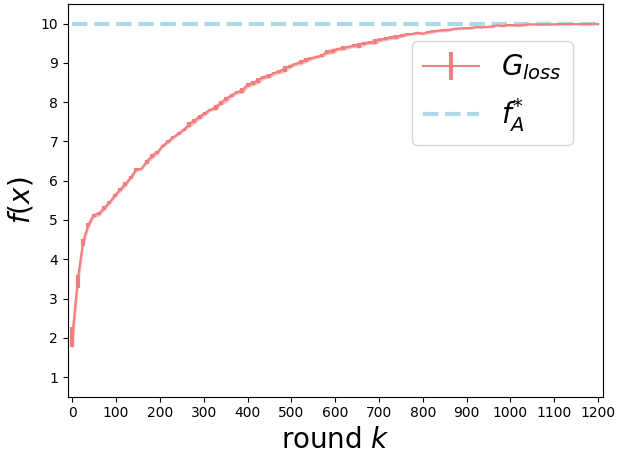}}
\subfigure[Visualizing sample generattion]{\includegraphics[width=0.26\textwidth]{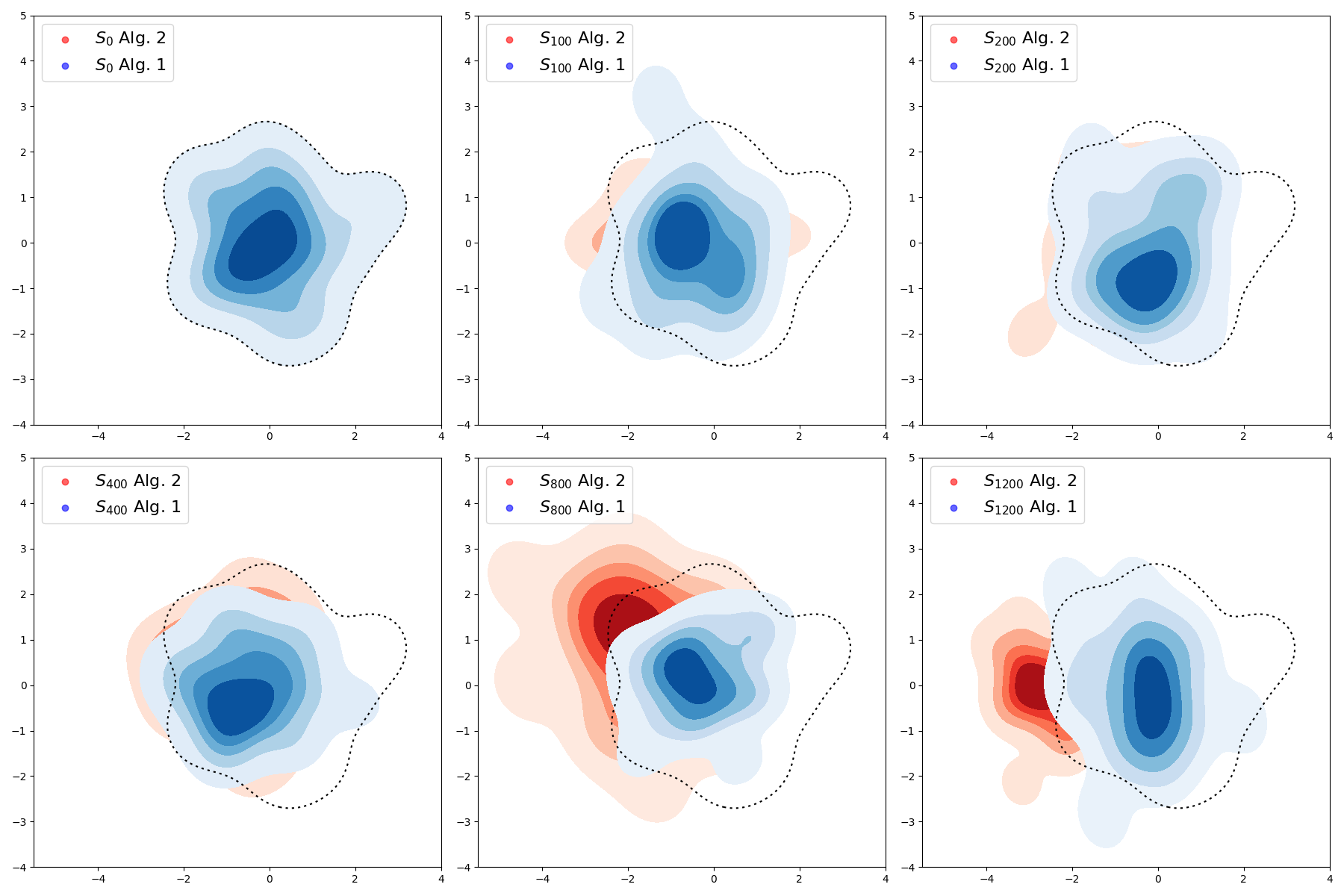}}
\vspace{-10pt}
\caption{\small{
\textbf{Convergence of Algorithms~\ref{alg:main} and \ref{alg:main_w_update}}. (a) and (b) are under different $\theta$ for the objective function. (d) visualizes the distribution of the generated samples of Alg.~\ref{alg:main_w_update} (red) and Alg.~\ref{alg:main} (blue) across the iterations.
}}
\vspace{-5pt}
\label{fig:convergence_simulation}
\end{figure}

\subsection{Image Generation}
We validate our theory in the image domain for \cref{alg:main}. We employ the StableDiffusion v1.5 model \citep{rombach2022high} as the pre-trained model. For the reward model, we follow the approach outlined by \cite{yuan2023reward} to construct a synthetic model. This model is based on a ResNet-18 \citep{he2016deep} architecture pre-trained on ImageNet \citep{deng2009imagenet}, with the final prediction layer replaced by a randomly initialized linear layer that produces scalar outputs.  For more experiment details, refer to \cref{sec:image_detail}. 
\begin{figure}[!htb]
\centering
\vspace{-12pt}
\subfigure{\includegraphics[width=0.6\textwidth, scale=0.3]{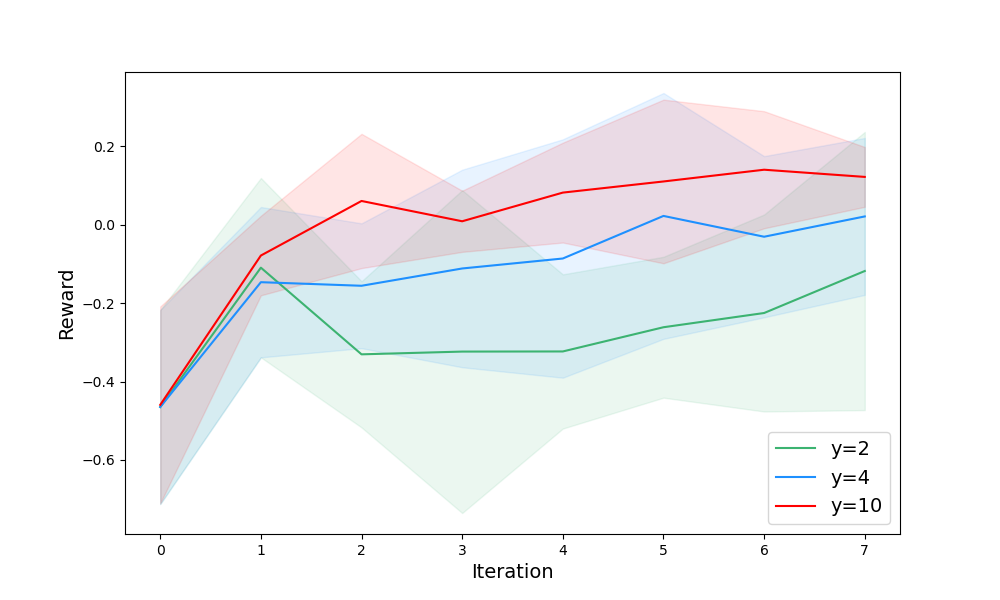}}
\subfigure{\includegraphics[width=0.36\textwidth, scale=0.3]{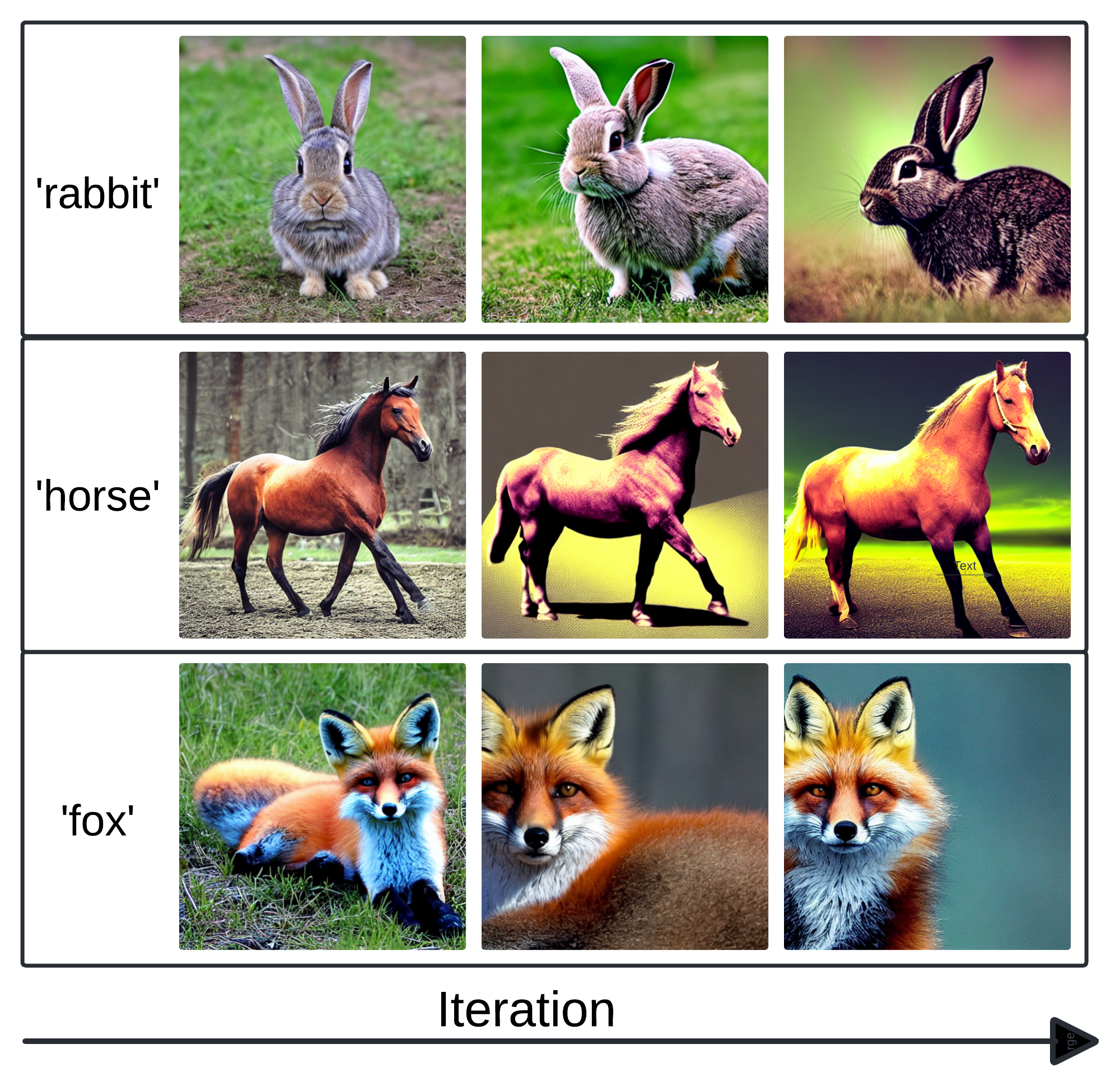}
}
\vspace{-10pt}
\caption
{\small{\textbf{Reward increase and effect on images across iterations.} Left: Reward increases and converges across iterations. Larger guidance strength $y$ (smaller regularizer strength) results in higher convergent reward value. Right: Images become more abstract, shifting from realistic to virtual backgrounds as reward increases.
}} 
\vspace{-12pt}
\label{fig:image_result}
\end{figure}

\vspace{-4pt}
\paragraph{Results.} By Algorithm \ref{alg:main}, the reward increases and converges. Figure \ref{fig:image_result} (left) shows the reward changes with optimization iterations. The hyperparameter $y$ tunes the strength of guidance and is inversely related to the strength of the regularizer (theoretical implications in \cref{appd:prof thms}). A larger guidance strength (smaller regularizer strength) leads to a higher convergent reward value. Figure \ref{fig:image_result} (right) illustrates the changes in generated images across iterations. As the reward increases, the images become increasingly abstract, shifting from photo-realistic with detailed backgrounds to more virtual, stylized ones.

\section{Conclusion}


In this paper, we focus on gradient guidance for adapting or fine-tuning pre-trained diffusion models from an optimization perspective. We investigate the look-ahead loss based gradient guidance and two variants of diffusion-based generative optimization algorithms utilizing it. We provide guarantees for adapting/fine-tuning diffusion models to maximize any target concave differentiable reward function. Our analysis extends to linear subspace data, where our gradient guidance and adaptive algorithms preserve and leverage the latent subspace, achieving faster convergence to near-optimal solutions.

\newpage
\appendix
\newpage
\section{Additional Related Works}
\label{sec:r_w}

\paragraph{{Direct Latent Optimization in Diffusion Models.}} Besides guidance methods, an alternative training-free route by optimizing the initial value of reverse process \cite{wallace2023end, ben2024d, karunratanakul2024optimizing, tang2024tuning, pan2023adjointdpm}. These methods typically backpropagate the gradient of reward directly to the initial latent vector through an ODE solver, utilizing the chain rule. Thus at inference time, the reverse process is unchanged except for being fed with an optimized initialization, different from the guidance method we studied.

\paragraph{Sampling and Statistical Theory of Diffusion Model} In contrast to the fruitful empirical advances, the theory of diffusion models is still limited. To the best of our knowledge, a theoretical understanding of fine-tuning diffusion models is absent. Existing results mainly focus on the sampling ability and statistical properties of unconditional diffusion models. In particular, for sampling ability, a line of works shows that the distribution generated by a diffusion model is close to the data distribution, as long as the score function is accurately estimated \citep{de2021diffusion, albergo2023stochastic, block2020generative, lee2022convergencea, chen2022sampling, lee2022convergenceb}. The accuracy of the estimated score function is measured in terms of an $L_\infty$ or $L_2$-norm distance.  More recently, \citep{chen2023restoration, chen2023probability, benton2023linear} develop refined and tighter analyses using Taylor expansions of the discretized backward process and localization method. It is worth mentioning that the analysis in \citep{chen2023restoration, chen2023probability, benton2023linear} extends to broader sample generation processes such as deterministic ones based on probabilistic ODEs. Going beyond distributions in Euclidean spaces, \citep{de2022convergence} analyzes diffusion models for sampling distribution supported on a low-dimensional manifold. Moreover, \citep{montanari2023posterior} consider sampling from symmetric spiked models, and \citep{el2023sampling} study sampling from Gibbs distributions using diffusion processes. 

Turning towards the statistical theory of diffusion models, \citep{song2020sliced} and \citep{liu2022let} provide asymptotic analyses, assuming a parametric form of the score function. Unfortunately, asymptotic analysis does not lead to concrete sample complexities. Later, concurrent works, \citep{oko2023diffusion} and \citep{chen2023score}, establish sample complexity bounds of diffusion models for estimating nonparametric data distributions. In high dimensions, their results highlight a curse of dimensionality issue without further assumptions, which also appears in \citep{wibisono2024optimal} considering kernel methods. More interestingly, these works demonstrate that diffusion models can circumvent the curse of dimensionality issue if the data has low-dimensional structures. In the same spirit, \citep{mei2023deep} investigate learning high-dimensional graphical models using diffusion models, without the curse of dimensionality. For conditional diffusion models, \citep{yuan2023reward, fu2024unveil} establish sample complexity bounds for learning generic conditional distributions. We refer readers to \citep{chen2024overview} for an overview of contemporary theoretical progress.

\section{Characterization for Output Distribution of Backward Process}\label{app:generated dis}
In this section, we provide analytical characterizations for the output distribution of the backward process guided by $\texttt{G}_{loss}$ when the pre-trained score is linear. We first give the result of score matching as follows.

\begin{lemma}[Pre-training with Linear Score Functions]\label{lmm:general score match}
Suppose for pre-training the score network, the class in \eqref{equ:score update objective} is
\begin{equation}
\tag{recall \eqref{equ:fully linear class}}
\cS = \cbra{s(x, t)=  C_t x + b_t : C_t \in \RR^{D\times D}, \, b_t \in \RR^D},
\end{equation}
If we freeze $C_t$ in \eqref{equ:fully linear class}, that is, minimizing the score matching objective \eqref{equ:score update objective} over the class $ \cbra{s(x, t)= C_t x + b_t : b_t \in \RR^D} $ gives 
\begin{equation*}
{s}_{\theta}(x_t,t)=  C_t\rbra{ x_t - \alpha(t)\bar{x}},
\end{equation*}
where $\Bar{x} = \rbra{\sum_{i=0}^k w_{k,i} {\EE}_{x\in \cD_i}[x]}\left /\rbra{\sum_{i=0}^k w_{k,i}} \right. $. Moreover, minimizing the score matching objective \eqref{equ:score update objective} over the class \eqref{equ:fully linear class} yields
\begin{equation*}
    {s}_{\theta}(x_t,t)= - \rbra{\alpha^2(t) \Bar{\Sigma} + h(t)I_D }^{-1}\rbra{ x_t - \alpha(t)\bar{x}},
\end{equation*}
where $\Bar{\Sigma} = \rbra{\sum_{i=0}^k w_{k,i} {\EE}_{x\in \cD_k} \sbra{(x-\Bar{x}) (x-\Bar{x})^\top}} \left / \rbra{\sum_{i=0}^k w_{k,i}} \right. $ are weighted data covariance.
\end{lemma}

\begin{proof}
Using the linear score network class $\cS$ with freezing $C_t$, we cast the score matching loss \eqref{equ:score update objective} into
\begin{align*}
& \quad \int_{0}^{T} \sum_{i=0}^k w_{k,i} {\EE}_{x_0 \in \mathcal{D}_i} \EE_{x_t\sim\cN\rbra{\alpha(t)x_0,h(t)I_D}}  \sbra{\dabs{\nabla_{x_t} \log \phi_t(x_t|x_0)-s(x_t, t)}^2} \dif t \\
& = \int_{0}^{T} \sum_{i=0}^k w_{k,i} {\EE}_{x_0 \in \mathcal{D}_i} \EE_{x_t\sim\cN\rbra{\alpha(t)x_0,h(t)I_D}}  \sbra{\dabs{-\frac{1}{h(t)} (x_t - \alpha(t)x_0) - C_tx_t - b_t}^2} \dif t \\
& = \int_{0}^{T} \sum_{i=0}^k w_{k,i} {\EE}_{x_0 \in \mathcal{D}_i} \EE_{x_t\sim\cN\rbra{\alpha(t)x_0,h(t)I_D}}  \sbra{\dabs{\left(C_t + \frac{1}{h(t)}I_D \right)(x_t - \alpha(t)x_0) + (\alpha(t)C_tx_0 + b_t)}^2} \dif t \\
& \overset{(i)}{=} \int_{0}^{T} \sum_{i=0}^k w_{k,i} {\EE}_{x_0 \in \mathcal{D}_i}  \left[\dabs{b_t + \alpha(t)C_tx_0}^2\right] + w \int_{0}^{T} \text{trace}\rbra{ h(t) \left(C_t + \frac{1}{h(t)}I_D\right)^\top\left(C_t + \frac{1}{h(t)}I_D\right) }\dif t,
\end{align*}
where $w = \sum_{i=0}^k w_{k,i}$, equality $(i)$ follows from computing the expectation over the conditional Gaussian distribution of $x_t | x_0$. We note that $b_t$ should minimize $\sum_{i=0}^k w_{k,i} {\EE}_{x_0 \in \mathcal{D}_i}  \left[\dabs{b_t + \alpha(t)C_tx_0}^2\right]$ for any $t$, which leads to
\begin{align*}
\hat{b}_t = -\alpha(t) C_t \bar{x}.
\end{align*}
Now,  we solve $C_t$ for the second result.
Substituting $\hat{b}_t$ into the optimization objective \eqref{equ:score update objective} yields:
\begin{equation*}
    \begin{aligned}
        & \quad  \int_{0}^{T} \sum_{i=0}^k w_{k,i} {\EE}_{x_0 \in \mathcal{D}_i} \left[\dabs{\alpha(t) C_t \bar{x} - \alpha(t)C_tx_0}^2\right] + w \int_{0}^{T} \text{trace}\rbra{  h(t) \left(C_t + \frac{1}{h(t)}I_D\right)^\top\left(C_t + \frac{1}{h(t)}I_D\right)}\dif t \\
        &= \int_{0}^{T} \alpha^2(t)\sum_{i=0}^k w_{k,i} {\EE}_{x_0 \in \mathcal{D}_i}  \left[\dabs{ C_t (x_0 - \bar{x})}^2\right] + w \int_{0}^{T} \text{trace}\rbra{ h(t) \left(C_t^\top C_t + \fs{1}{h(t)}C_t +  \fs{1}{h(t)}C_t^\top + \fs{1}{h^2(t)}I_D\right)}\dif t. \\
    \end{aligned}
\end{equation*}
Taking the gradient for $C_t$, we get
\begin{equation*}
    2 \alpha^2(t) C_t \sum_{i=0}^k w_{k,i} {\EE}_{x_0 \in \mathcal{D}_i} \sbra{(x_0 - \bar{x}) (x_0 - \bar{x})^\top} + 2w h(t) C_t + 2 wI_D.
\end{equation*}
Setting the gradient above to 0, we get the solution for $ C_t$ as
\begin{equation*}
    \hat{C}_t = -\rbra{ \alpha^2(t) w^{-1} \sum_{i=0}^k w_{k,i} {\EE}_{x_0 \in \mathcal{D}_i}\sbra{(x_0 - \bar{x}) (x_0 - \bar{x})^\top} + h(t)I_D}^{-1}.
\end{equation*}
Therefore, the proof is completed.

\end{proof}

When $w_{k,0} = 1, w_{k,i} =  0, i  \in [k] $, Lemma~\ref{lmm:general score match} reduces to the pre-traning score matching.

\begin{corollary}
\label{coro:pretrain score matching}
Let $\cD$ be the pre-training data. Minimizing the score matching objective \eqref{equ:score match objective} over the function class \eqref{equ:fully linear class} gives 
\begin{equation}\label{equ:pre-train score solution linear}
{s}_{\theta}(x_t,t)= - \rbra{\alpha^2(t) \Bar{\Sigma} + h(t)I_D }^{-1}\rbra{ x_t - \alpha(t)\bar{\mu}}.
\end{equation}
\end{corollary}

The following lemma characterizes the output distribution of the backward process guided by $\texttt{G}_{loss}$ when the pre-trained score is linear.

\begin{lemma} \label{lmm:posterior distribution general}
If the pre-trained score $s_\theta (x_t,t)$ is \eqref{equ:pre-train score solution linear}, substituting the score function with $s_\theta (x_t,t) + \texttt{G}_{loss}(x_t, t) $ in the backward SDE \eqref{eq:backward} yields, when $T \to \infty$,
\begin{equation*}
    X_{T}^{\leftarrow} \overset{d}{=}  \cN\left(\bar{\mu} +  \frac{y - g^\top \mu}{\sigma^2 + g^\top \Bar{\Sigma} g} \Bar{\Sigma}  g,  \Bar{\Sigma} -  \frac{\Bar{\Sigma} g   g^\top \Bar{\Sigma}}{\sigma^2 + g^\top \Bar{\Sigma} g}  \right).
\end{equation*}
with $\beta(t)$ assigned as  $\beta(t) = \fs{1}{2} \rbra{\sigma^2 + g^\top \Bar{\Sigma}^{-1}  \rbra{ I_D +  \alpha^2(t)\Bar{\Sigma}/h(t) }^{-1} g     }^{-1}$.  Moreover, if pre-training data reside in $\operatorname{Span}(A)$ following  Assumption~\ref{asmp:subspace_data}, it holds $X_{T}^{\leftarrow} \in \operatorname{Span}(A)$.
\end{lemma}

\begin{proof}
Consider $X_0 \overset{d}{=} \cN \rbra{ \bar{\mu}, \Bar{\Sigma}}, Y=g^\top X_0 + \epsilon $ where $\epsilon \sim \cN \rbra{0, \sigma^2 }$. Let $X_0$ be the initialization of the forward process. Similar to the proof in \cref{prf:lmm:guidance=cond w.o.subspace},
we get 
\begin{align*}
\begin{bmatrix}
X_t \\
Y
\end{bmatrix}
\overset{d}{=} \cN\left(
\begin{bmatrix}
\alpha(t) \bar{\mu} \\
g^\top \bar{\mu}
\end{bmatrix},
\begin{bmatrix}
\alpha^2(t) \Bar{\Sigma} + h(t) I_D & \alpha(t) \Bar{\Sigma} g \\
\alpha(t)g^\top \Bar{\Sigma} & \sigma^2 + g^\top \Bar{\Sigma} g
\end{bmatrix}\right).
\end{align*}
Thus, we get $s_\theta (x_t,t)$ is exactly the score of marginal distribution of $X_t$, i.e., $\nabla \log p_t(x_t) =  s_\theta (x_t,t)$. According to the proof in \cref{prf:lmm:guidance=cond w.o.subspace}, we get $ s_\theta (x_t,t) + \texttt{G}_{loss}(x_t, t) = \nabla \log p_t(x_t\mid y) $. Thus, the backward SDE turns out to be
\begin{equation}\label{equ:backward catch condi score}
    \diff \bXb_t = \left[\frac{1}{2} \bXb_t + \nabla \log p_{T-t}(\bXb_t\mid y)\right] \diff t + \diff \overline{W}_t, \quad \bXb_0 \overset{d}{=} \cN \rbra{0, I_D}.
\end{equation}
The initial distribution $p_0(x_0 \mid y)$ of the forward process can also be obtained by \eqref{equ:backward catch condi score} where we replace the initial distribution as $p_T(x_T\mid y)$. According to the data processing inequality, we get the bound of the total variation distance between the terminal distribution $p_{T}^{\leftarrow} $ of \eqref{equ:backward catch condi score} and $p_0(x_0 \mid y)$:
\begin{equation*}
    \text{TV}\rbra{p_0 , p_{T}^{\leftarrow}  } \leq \text{TV}\rbra{p_T , \varphi  },
\end{equation*}
where $p_0, p_T$ are short hands for $p_0(x_0 \mid y) $ and 
$p_T(x_T\mid y)$, and $\varphi (\cdot)$ is the density for the standard normal distribution $\cN \rbra{0, I_D}$. Since in the forward process, $p_T \to \varphi$ when $T \to \infty $, we have $\text{TV}\rbra{p_0 , p_{T}^{\leftarrow}  } \to 0 $ when $T \to \infty$. We complete the first part of the lemma. As for the second part, if data reside in $\operatorname{Span}(A)$ following Assumption~\ref{asmp:subspace_data}, we have $\bar{\mu} = A\bar{u}  $ and $\Bar{\Sigma}  = A \bar{\Sigma}_{u} A^\top $, where $\bar{u} = {\EE}_{x \in \cD, u = A^{\top}x}[u] $, $ \Bar{\Sigma}_{u} = {\EE}_{x \in \cD, u = A^{\top}x}[\rbra{u-\bar{u}}\rbra{u-\bar{u}}^\top] $. Thus, the covariance matrix of $\bXb_T$ is
\begin{equation*}
    \Bar{\Sigma} -  \frac{\Bar{\Sigma} g   g^\top \Bar{\Sigma}}{\sigma^2 + g^\top \Bar{\Sigma} g} = A \sbra{\bar{\Sigma}_{u} -  \frac{\bar{\Sigma}_{u} A^\top g   g^\top A \bar{\Sigma}_{u}
}{\sigma^2 + g^\top \Bar{\Sigma} g} } A^\top,
\end{equation*}
and due to $\bXb_T$ follows Gaussian distribution, we get  $X_{T}^{\leftarrow} \in \operatorname{Span}(A)$.
Thus, the proof is completed.
\end{proof}

\section{Additional Materials for \cref{sec:grad_guidance}}
\DoToC

\subsection{Score decomposition for subspace data} \label{sec:score_decomposition}
Under \cref{asmp:subspace_data}, the score function $\nabla \log p_t(x)$ decomposes to two orthogonal parts: an on-support component belonging to the subspace; and an orthogonal component. We recall this key result in Proposition \ref{prop:score dcp}, which later plays a key role in deriving subspace preserving guidance.

\begin{proposition}
[Score Decomposition for Subspace Data (\cite{chen2023score} Lem.~1, Thm.~3)] 
\label{prop:score dcp}
Under \cref{asmp:subspace_data}, the score function $\nabla \log p_t(x)$ decomposes as
\begin{align}
\label{equ:decomposition}
\nabla \log p_t(x) = \underbrace{A\nabla \log p_t^{\sf LD}(A^\top x)}_{\pscore(A^\top x, t) \text{:~on-support~score}}  
-
\underbrace{ {h^{-1}(t)} \left(I_D - AA^\top \right) x}_{\oscore(x, t) \text{:~ortho.~score}}.
\end{align}
where
$
p_t^{\sf LD}(u^{\prime}) =  \int \phi_t(u^{\prime}|u)p_u(u) \diff u
$
with $\phi_t( \cdot | u)$ being the density of $\cN (\alpha(t)u, h(t)I_d)$ for the same  $\alpha(t)$ and $h(t)$ in the forward process \eqref{eq:forward_sde}.
\end{proposition}

\subsection{Proof of \cref{prop:failure_naive}}
We give the proof of \cref{prop:failure_naive}, which shows the failure of naive gradient guidance.
\begin{proof}
    Under \cref{asmp:subspace_data}, the score can be decomposed to terms parallel and orthogonal to $\operatorname{Span}(A)$ (\cref{prop:score dcp}). Applying naive guidance, we examine the orthogonal reverse process:
$$
\mathrm{d} X_{t, \perp}^{\leftarrow} =\left[\frac{1}{2}-\frac{1}{h(T-t)}\right] X_{t, \perp}^{\leftarrow}\mathrm{d} t + b(t)g \mathrm{d} t+\left(I_D-A A^\top\right) \mathrm{d} \overline{W}_t.
$$
Solving this SDE, we get the expectation of the final state following $\mathbb{E}[X_{T, \perp}^{\leftarrow}] = \int_0^T \exp \left(- \int_0^{t}h^{-1}(s)\mathrm{d}s \right) e^{t/2}b(T-t) g \mathrm{d}t$. For the schedule $h(t) = 1 - \exp(-\sqrt{t})$, we have the coefficient of direction $g$ is larger than $\int_0^T \exp(-t/2-2\sqrt{t})b(T-t)\mathrm{d}t > \int_{0}^1 \exp(-5/2)b_0 \mathrm{d}t >0$ where we can assume $T>1$. Thus, $\mathbb{E}[ X_{T, \perp}^{\leftarrow}] \neq 0$. This means the generated sample is leaving the subspace, i.e., naive gradient guidance will violate the latent structure.
\end{proof}

\subsection{Discussion for gradient-like guidance}
To further clarify the derivation of our gradient guidance, we present a variant of Lemma~\ref{lmm:guidance=cond w.o.subspace}.
\begin{proposition}\label{prop:g=scaler_grad}
Under Assumption~\ref{asump: gaussain data linear reward}, we have
\begin{equation}
\label{equ:guidance_gaussian_appd}
    \nabla_{x_t} \log p_t(y|x_t) = \beta(t) \sbra{y - g^{\top} \EE[x_0 | x_t]} \cdot \rbra{\alpha^2(t)\Sigma + h(t) I_D }^{-1} \Sigma g,
\end{equation}
    where $\EE[x_0 | x_t] $ denotes the conditional expectation of $x_0$ given $x_t$ in the forward process \eqref{eq:forward_sde},  and $\beta(t) =  \alpha(t)/ (\sigma^2 + g^\top \Sigma^{-1}  \rbra{ I_D + \alpha^2(t)/h(t) \cdot \Sigma }^{-1} g ) $.
\end{proposition}

\paragraph{Remarks.} Observe that, when $\Sigma=I$, \eqref{equ:guidance_gaussian_appd} suggests the following form of guidance that is aligned with the naive gradient, i.e., the steepest ascent direction:
$$ \texttt{G}(x_t, t) \propto \sbra{y - g^{\top} \EE[x_0 | x_t]} \cdot  g.$$

However, even for Gaussian distributions, as long as $\Sigma \neq I$, the term of \eqref{equ:guidance_gaussian} is no longer proportional to $g$ but becomes a pre-conditioned version of the gradient. We show the guidance above can maintain the subspace structure of data in the experiments \cref{sec:experiment details}.

\begin{wrapfigure}{r}{0.35\textwidth}
\centering
\vspace{-20pt}
\includegraphics[width=0.35\textwidth]
{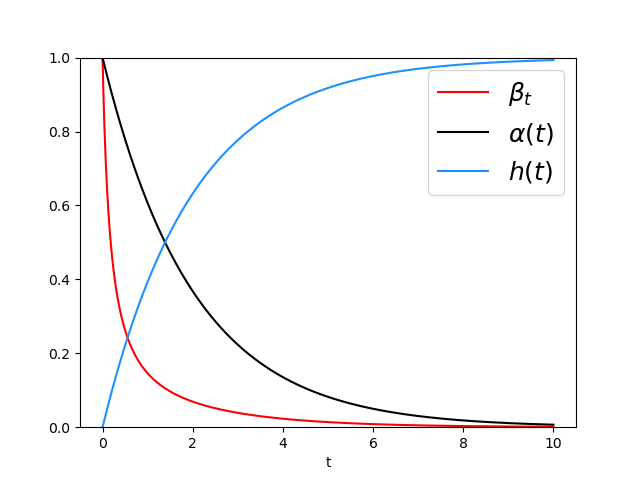}
\vspace{-20pt}
\caption{\small \textbf{Plot of $\beta(t), \alpha(t), h(t)$ for $t \in [0,10]$} when $\Sigma = I$.}
\label{fig:beta}
\end{wrapfigure}

Another observation is that this guidance scales with a residual term 
$y - g^{\top} \EE[x_0 \mid x_t]$. In particular, the residual term $y - g^{\top} \EE[x_0 \mid x_t]$ tunes the {\it strength of guidance}. Recall $\EE[x_0 \mid x_t]$ denotes the posterior expectation of clean data $x_0$ given $x_t$ in the forward process. Thus, in a backward view, $\EE[x_0 \mid x_t]$ coincides with the expected sample to be generated conditioned on $x_t$. In this sense, the quantity $y - g^{\top} \EE[x_0 \mid x_t]$ measures a \textbf{look-ahead gap} {between the expected reward of generated samples and the target value}. A larger absolute value of the residual means stronger guidance in the backward generation process. 

Under $q(t) \equiv 1$ in the forward process \eqref{eq:forward_sde},  we plot the theoretical choice of $\beta(t)$ and $\alpha(t), h(t)$ to $t$ in \cref{fig:beta}. In practice, the choice of $\alpha(t), h(t)$ can vary and they are determined by the forward process used for pre-training; and $\beta(t)$ can be treated as a tuning parameter to adjust the strength of guidance. 

The proof of \cref{prop:g=scaler_grad} is included in the proof of Lemma~\ref{lmm:guidance=cond w.o.subspace}.

\subsection{Proof of Lemma~\ref{lmm:guidance=cond w.o.subspace}}\label{prf:lmm:guidance=cond w.o.subspace}

\begin{proof}
Recall $\cbra{X_t}_{t \ge 0} $ is the stochastic process from the forward process. $X_0 \overset{d}{=} \cN \rbra{\mu, \Sigma }, Y = g^\top X_0 + \epsilon $ where $\epsilon \sim \cN (0, \sigma^2)$ is independent with $X_0$.
Since $X_0, X_t$ and $Y$ are joint Gaussian, we have $Y \mid X_t$ also follows Gaussian distribution, denoted as $\cN \rbra{m_y(x_t), \sigma_y^2(x_t)}$.Then, the closed form of $\nabla_{x_t} \log p_t(y \mid x_t)  $ can be derived  as
\begin{equation*}
    \nabla_{x_t} \log p_t(y \mid x_t) = - \nabla_{x_t} \sbra{\fs{1}{2}\rbra{\frac{y - m_y(x_t) }{\sigma_y(x_t)}}^2 } - \nabla_{x_t} \log \sigma_y(x_t).
\end{equation*}
Due to the linearity of $Y$ with regard to $X_0$, $m_y(x_t)$ can be computed as
\begin{equation}\label{equ:mean of y on xt}
    \begin{aligned}
        m_y(x_t) = \EE [y \mid x_t] = \EE [g^\top x_0 + \epsilon \mid x_t] = \EE [g^\top x_0  \mid x_t] = g^\top \EE [ x_0  \mid x_t]. 
    \end{aligned}
\end{equation}
To get the variance $\sigma_y^2(x_t)$, we compute the joint distribution $(X_t, Y)$. In the forward process, given $X_0 = x_0$, $X_t$ can be written as $\alpha(t) x_0 + Z_t$ for $Z_t \overset{d}{=} \cN(0, h(t)I_D)$ independent of $x_0$. Due to the linear function assumption, we have
\begin{align*}
\begin{bmatrix}
X_t \\
Y
\end{bmatrix}
=
\begin{bmatrix}
\alpha(t)I_D & 0 & I_D \\
g^\top & 1 & 0
\end{bmatrix}
\begin{bmatrix}
x_0 \\
\epsilon \\
Z_t
\end{bmatrix}.
\end{align*}
Observing that the joint distribution of $(x_0, \epsilon, z_t)$ is Gaussian, we deduce
\begin{align*}
\begin{bmatrix}
X_t \\
Y
\end{bmatrix}
\overset{d}{=} \cN\left(
\begin{bmatrix}
\alpha(t) \mu \\
g^\top \mu
\end{bmatrix},
\begin{bmatrix}
\alpha^2(t) \Sigma + h(t) I_D & \alpha(t) \Sigma g \\
\alpha(t)g^\top \Sigma & \sigma^2 + g^\top \Sigma g
\end{bmatrix}\right).
\end{align*}
Thus, we get $\sigma_y^2(x_t) = \sigma^2 + g^\top \Sigma g - \alpha^2(t) g^\top \Sigma \rbra{\alpha^2(t) \Sigma + h(t) I_D }^{-1} \Sigma g  $. Together with the derivation of the mean $m_y(x_t)$ \eqref{equ:mean of y on xt}, we get
\begin{equation*}
\begin{aligned}
    \nabla_{x_t} \log p_t(y \mid x_t) &= - \fs{1}{2 \sigma_y^2(x_t)  } \nabla_{x_t} \sbra{\rbra{y - g^\top \EE [ x_0  \mid x_t]}^2 }  \\
    &= \fs{1}{ \sigma_y^2(x_t)  } \rbra{y - g^\top \EE [ x_0  \mid x_t] } \nabla_{x_t} \EE [ x_0  \mid x_t] g.
\end{aligned}
\end{equation*}
To get $\EE [ x_0  \mid x_t]$, we derive the joint distribution of $(X_0, X_t)$:
\begin{align*}
\begin{bmatrix}
X_0 \\
X_t
\end{bmatrix}
\overset{d}{=} \cN\left(
\begin{bmatrix}
\mu \\
\alpha(t)\mu
\end{bmatrix},
\begin{bmatrix}
\Sigma & \alpha(t)\Sigma \\
\alpha(t)\Sigma & \alpha^2(t)\Sigma + h(t) I_D
\end{bmatrix}
\right).
\end{align*}
Thus, we get $\EE[x_0 | x_t] =  \mu + \alpha(t)  \rbra{\alpha^2(t)\Sigma + h(t) I_D }^{-1} \Sigma \rbra{x_t - \alpha(t)\mu} $. As a consequence, we have
\begin{equation*}
    \nabla_{x_t} \log p_t(y \mid x_t) = \fs{1}{ \sigma_y^2(x_t)  } \rbra{y - g^\top \EE [ x_0  \mid x_t] } \alpha(t)  \rbra{\alpha^2(t)\Sigma + h(t) I_D }^{-1} \Sigma g.
\end{equation*}
Together with the following equality by Woodbury identity, we get the result.
\begin{equation*}
\begin{aligned}
      \sigma_y^{-2}(x_t) &= \rbra{\sigma^2 + g^\top \Sigma g - \alpha^2(t) g^\top \Sigma \rbra{\alpha^2(t) \Sigma + h(t) I_D }^{-1} \Sigma g}^{-1} \\
    &= \sbra{\sigma^2 + g^\top \Sigma^{-1}  \rbra{ I_D + \fs{\alpha^2(t)}{h(t)} \Sigma }^{-1} g }^{-1}.
\end{aligned}
\end{equation*}
\end{proof}

\subsection{Proof of Theorem \ref{thm:faith_grad_guidance}}\label{prof:faith_grad_guidance}

\begin{proof}
We expand the derivative in $\texttt{G}_{loss}$ as
\begin{align*}
\texttt{G}_{loss}(x_t, t) = 2 \beta(t) (y - g^\top \EE[x_0 | x_t]) \left(\nabla_{x_t} {\EE}[x_0 | x_t]\right)^\top g.
\end{align*}
It holds ${\EE}[x_0 | x_t] = \frac{1}{\alpha(t)}(x_t + h(t) \nabla \log p_t(x_t))$. 
Via the score decomposition under linear subspace data in \citet[Lemma 1]{chen2023score}, we have
\begin{align*}
\nabla \log p_t(x_t) = A \nabla \log p_t^{\sf LD}(A^\top x_t) - \frac{1}{h(t)} \left(I_D - AA^\top \right) x_t,
\end{align*}
where $p_t^{\sf LD}$ denotes the diffused latent distribution, i.e., $p_t^{\sf LD}(u') = \int \phi_t(u'|u_0) p_u(u_0) \dif u_0$. Recall that $\phi_t$ is the Gaussian transition kernel of the forward process and $p_u$ is the density of latent variable $u_0$ in Assumption~\ref{asmp:subspace_data}.

To ease the derivation, we denote $m(u) = \nabla \log p_t^{\sf LD}(u) + \frac{1}{h(t)}u$. It then holds that
\begin{align*}
{\EE}[x_0 | x_t] & = \frac{1}{\alpha(t)} \left(x_t + h(t) \left[Am(A^\top x_t) - \frac{1}{h(t)}x_t\right]\right) \\
& = \frac{h(t)}{\alpha(t)}A m(A^\top x_t).
\end{align*}
As a consequence, we can verify that
\begin{align*}
\nabla_{x_t} {\EE}[x_0 | x_t] = \frac{h(t)}{\alpha(t)} A \left[\nabla m(A^\top x_t)\right] A^\top,
\end{align*}
where $\nabla m(A^\top x_t) \in \RR^{d \times d}$ is the Jacobian matrix of $m$ at $A^\top x_t$. Plugging the last display into $\texttt{G}_{loss}$, we conclude that 
\begin{align*}
\texttt{G}_{loss}(x_t, t) & = 2 \beta(t) (y - g^\top {\EE}[x_0 | x_t]) \left(\nabla_{x_t} {\EE}[x_0 | x_t]\right)^{\top} g \\
& = \ell_t \cdot g
\end{align*}
for $\ell_t = 2 \beta(t) (y - g^\top {\EE}[x_0 | x_t])$ and $g' = \left(\nabla_{x_t} {\EE}[x_0 | x_t]\right)^{\top} g = \frac{h(t)}{\alpha(t)} A \left[\nabla m(A^\top x_t)\right]^\top A^\top g \in {\rm Span}(A)$. The proof is complete.
\end{proof}

\section{Additional Materials for in \cref{sec:reg opt,sec:update pre-score}}\label{appd:prof thms}
\DoToC

\subsection{Convergence to Regularized Maxima in Mean}\label{pf:thm full space}
Here we present a more fundamental version of Theorem~\ref{thm:subspace fully linear}.

\begin{theorem}[Convergence to Regularized Maxima {in Mean}] \label{thm:fully linear}
Let Assumption~\ref{asmp:f} hold, and let the pre-training data $\mathcal{D}$ have arbitrary distribution with covariance matrix $\Bar{\Sigma} \succ {\rm 0} $. Suppose the score function $s_\theta$ is pre-trained via minimizing the score matching loss \eqref{equ:score match objective} over the linear function class \eqref{equ:fully linear class}. Let  Alg.~\ref{alg:main} take ${s}_{\theta}(\cdot,\cdot)$ and $f$ as the input. For any $\lambda > L,$ there exists $\{\beta(t)\}$, $ \cbra{y_k}$, $\{B_k\}$ such that, with probability $\geq 1 - \delta$ , the mean of the output distribution $ \mu_K$ converges to be near $x^*_{\lambda}$, and 
\begin{equation}\label{equ:conv rate fully linear}
    f\rbra{x^*_{\lambda}} - f(\mu_K)  =  \lambda \rbra{ \fs{L}{\lambda}}^K \mathcal{O}\rbra{D \log \rbra{\fs{K}{\delta}}}, 
\end{equation}
where $D$ is the ambient dimension of data, and $x^*_{\lambda}$ is a regularized maximizer of $f$ given by
\begin{equation}
    \label{equ:reg_obj w. covariance}
    x^*_{\lambda} = \argmax_{x\in \RR^D}~\left\{ f(x)  - \fs{\lambda}{2}\dabs{x-\bar{\mu}}^2_{\Bar{\Sigma}^{-1}}\right\},
\end{equation}
where $\bar \mu,\bar\Sigma$ are empirical mean and covariance of pre-training data $\mathcal{D}.$
\end{theorem}

\begin{proof}
The proof is a special case of Theorem~\ref{thm:subspace fully linear} in Appendix~\ref{pf:subspace fully linear}, via setting the representation matrix $A = I_D$.
\end{proof}

\subsection{Proof of Theorem~\ref{thm:subspace fully linear}} \label{pf:subspace fully linear}

We first provide a proof sketch.

\paragraph{Proof Sketch} Solving the score matching problem \eqref{equ:score match objective} with a linear function class \eqref{equ:fully linear class} yields a pre-trained score as follows
\begin{equation*}
{s}_{\theta}(x_t,t)= - \rbra{\alpha^2(t) \Bar{\Sigma} + h(t)I_D }^{-1}\rbra{ x_t - \alpha(t)\bar{\mu}}.
\end{equation*}
With proper choices of $\beta(t)$, gradient guidance $\texttt{G}_{loss}$ leads to the following output distribution at the end of round $k$:  
\begin{equation*}
     \cN\left(\bar{\mu} +  \frac{y_k - g_k^\top \bar{\mu}}{\sigma^2 + g_k^\top \Bar{\Sigma} g_k} \Bar{\Sigma}  g_k,  \Bar{\Sigma} -  \frac{\Bar{\Sigma} g_k   g_k^\top \Bar{\Sigma}}{\sigma^2 + g_k^\top \Bar{\Sigma} g_k}  \right).
\end{equation*}
Thus, we obtain the mean of the above distribution, i.e., $\mu_{k+1} = \bar{\mu} + \eta_{k} \Bar{\Sigma} \nabla f(\bar{z}_{k})$, where $\bar{z}_k$ is the empirical mean of previous samples, $\eta_k$ is a stepsize determined by $y_k$. 
By a rearrangement, we obtain a recursive formula
\begin{equation}
\label{equ:up_rule_GD}
\mu_{k+1} = \bar{z}_{k} + \eta_k \Bar{\Sigma} \sbra{ \nabla f(\bar{z}_{k}) - \eta_k^{-1} \Bar{\Sigma}^{-1} \rbra{\bar{z}_{k}-\bar{\mu}}}.
\end{equation}
We observe that \eqref{equ:up_rule_GD} resembles a gradient ascent update from $\mu_{k} \approx \bar z_k$ to $\mu_{k+1}$ corresponding to a regularzed optimization problem \eqref{equ:reg_obj w. covariance}. In this regularized objective,  the original objective $f(x)$ incorporates an additional proximal term with $\lambda := 1/\eta_k$. 
Therefore we can analyze the convergence of $\mu_k$ by following the classical argument for gradient optimization.

\begin{proof}[Proof of Theorem~\ref{thm:subspace fully linear}]
Define a filtration $\cbra{\cH_k}_{k=0}^{K-1}$ with $\cH_k$ be the information accumulated after $k$
rounds of Alg.\ref{alg:main}.
\begin{equation*}
\begin{aligned}
      \cH_0 &:= \sigma (\bar{\mu}), \\
      \cH_k &:= \sigma \rbra{\cH_{k-1}, \sigma \rbra{z_{k-1,1}, \ldots, z_{k-1,B_{k-1}}} }, \quad k \in [K].
\end{aligned}
\end{equation*}
Define the expectation of samples generated at $k$-th round as
\begin{equation*}
    \mu_k := \EE [z_{k,i} \mid \cH_{k-1}], \quad k \in [K-1].
\end{equation*}
 Applying \cref{coro:pretrain score matching}, we get the pre-trained score as
\begin{equation*}
    {s}_{\theta}(x_t,t)= - \rbra{\alpha^2(t) \Bar{\Sigma} + h(t)I_D }^{-1}\rbra{ x_t - \alpha(t) \bar{\mu}},
\end{equation*}
If we set $y_k$ as follows 
\begin{equation*}
    y_k = \eta \cdot \rbra{\sigma^2 + g_k^\top \Bar{\Sigma} g_k } + g_k^\top  \bar{\mu},
\end{equation*}
where $\eta = 1/ \lambda$. 
And we choose $\beta(t)$ at  $k$-round as
$\beta(t) = \fs{1}{2} \rbra{\sigma^2 + g_{k-1}^\top \Bar{\Sigma}^{-1}  \rbra{ I_D +  \alpha^2(t)\Bar{\Sigma}/h(t) }^{-1} g_{k-1}}^{-1}$.
Then, Lemma~\ref{lmm:posterior distribution general} provides the generated distribution in $k$-th round:
\begin{equation}\label{equ:generated_gaussian}
    \cN \rbra{\bar{\mu} + \eta \Bar{\Sigma} g_{k-1}, \Bar{\Sigma} -  \frac{\Bar{\Sigma} g_{k-1}   g_{k-1}^\top \Bar{\Sigma}}{\sigma^2 + g_{k-1}^\top \Bar{\Sigma} g_{k-1}} }.
\end{equation}
 Define  the empirical covariance matrix of the latent variable $U$ as $ \Bar{\Sigma}_{u} = {\EE}_{x \in \cD, u = A^{\top}x}[\rbra{u-\bar{u}}\rbra{u-\bar{u}}^\top] $ where $\bar{u} = {\EE}_{\cD}[u] $. Then in the subspace setting, the empirical mean and covariance of data $X$ can be written as $ \bar{\mu} = AA^\top \bar{\mu}$ and $ \Bar{\Sigma} = A\Bar{\Sigma}_{u}A^\top $ respectively. The mean of the sample $z_{k,i}$ follows
\begin{equation*}
   \mu_k =  \EE [z_{k,i} \mid \cH_{k-1}] = AA^\top \bar{\mu} + \eta  \cdot A \Bar{\Sigma}_{u} A^\top   g_{k-1},
\end{equation*}
where $g_{k-1} = \nabla f \rbra{\bar{z}_{k-1}}$ and $\bar{z}_{k-1} =  \rbra{1/B} \sum_i^B z_{k-1, i} $. We rearrange the update rule to show a gradient ascent formula as follows
\begin{equation*}
\begin{aligned}
   \mu_k &= AA^\top \mu_{k-1} - AA^\top \rbra{\mu_{k-1} - \bar{\mu} } +  \eta  \cdot A \Bar{\Sigma}_{u} A^\top   \nabla f(\mu_{k-1}) + \eta  \cdot A \Bar{\Sigma}_{u} A^\top  \rbra{ g_{k-1}-\nabla f(\mu_{k-1})}\\
    &= AA^\top \mu_{k-1} - A\Bar{\Sigma}_{u} A^\top A\Bar{\Sigma}_u^{-1} A^\top \rbra{\mu_{k-1} - \bar{\mu} } +  \eta  \cdot A \Bar{\Sigma}_{u} A^\top   \nabla f(\mu_{k-1}) + \eta  \cdot A \Bar{\Sigma}_{u} A^\top  \rbra{ g_{k-1}-\nabla f(\mu_{k-1})} \\
    &= AA^\top \mu_{k-1} + \eta  \cdot A \Bar{\Sigma}_{u} A^\top \sbra{ \nabla f(\mu_{k-1}) - \lambda A \Bar{\Sigma}_u^{-1} A^\top \rbra{ \mu_{k-1} - \bar{\mu}}  }  + \eta  \cdot A \Bar{\Sigma}_{u} A^\top  \rbra{ g_{k-1}-\nabla f(\mu_{k-1})}.
\end{aligned}
\end{equation*}
where $\lambda = 1/ \eta$. Define $h(x) := f(x) - \lambda / 2 \dabs{x - \bar{\mu}}^2_{\Bar\Sigma^{-1}} $, we have
\begin{equation*}
\begin{aligned}
   \mu_k &= AA^\top \mu_{k-1} + \eta  \cdot \Bar\Sigma  \nabla h(\mu_{k-1})  + \eta  \cdot \Bar\Sigma \rbra{ g_{k-1}-\nabla f(\mu_{k-1})}.
\end{aligned}
\end{equation*}
Recall the notation for the optimum: $x^\star_{A,\lambda} = \argmax_{x=Au} h(x)$. 
We consider the distance of $\mu_k$ to $x^\star_{A,\lambda}$ under the semi-norm $ \dabs{ \cdot}_{\Bar\Sigma^{-1}}$.
\begin{align}
    \dabs{\mu_k  - x^\star_{A,\lambda} }_{\Bar\Sigma^{-1}} &= \dabs{\mu_{k-1} - x^\star_{A,\lambda} + \eta \Bar\Sigma  \nabla h(\mu_{k-1}) + \eta \Bar\Sigma  \rbra{ g_{k-1}-\nabla f(\mu_{k-1})}}_{\Bar\Sigma^{-1}} \nonumber \\
    & \leq  \underbrace{\dabs{\mu_{k-1} - x^\star_{A,\lambda} + \eta \Bar\Sigma  \nabla h(\mu_{k-1})}_{\Bar\Sigma^{-1}}}_{:= I_1} + \underbrace{\dabs{ \eta \Bar\Sigma  \rbra{ g_{k-1}-\nabla f(\mu_{k-1})}}_{\Bar\Sigma^{-1}}}_{:= I_2}. \label{equ:distance two terms to bound Sigma}
\end{align}
We bound the second term $I_2$ first. According to $f$ is $L$-smooth with respect to $ \dabs{  \cdot}_{\Bar\Sigma^{-1}} $, we have 
\begin{equation*}
    \begin{aligned}
         I_2 &= \eta \dabs{  g_{k-1}-\nabla f(\mu_{k-1})}_{\Bar\Sigma} \leq \eta L  \dabs{ \bar{z}_{k-1}-\mu_{k-1}}_{\Bar\Sigma^{-1}},
    \end{aligned}
\end{equation*} Lemma~\ref{lmm:posterior distribution general} shows the distribution of $z_{k-1,i}$. Therefore, according to concentration inequality for Gaussian distribution, with the probability at least $1 - \delta / K $, it holds $$\dabs{ \bar{z}_{k-1} - \mu_{k-1}}^2_{\Bar\Sigma^{-1}} \leq 2 \log \rbra{\frac{2K}{\delta}} \cdot \fs{\text{trace}\rbra{\VV({z}_{k-1,i}) \cdot \Bar\Sigma^{-1}}}{B_{k-1}}. $$ 
We have $\text{trace}\rbra{\VV({z}_{k-1,i}) \cdot \Bar\Sigma^{-1} } \leq \text{trace}\rbra{\Bar{\Sigma} \cdot \Bar\Sigma^{-1} } = d $. Therefore, $I_2$ is bounded by $$I_2 \leq M_0 /\sqrt{B_{k-1}}, $$ where $M_0:=\eta L \sqrt{ 2 \log \rbra{\frac{2K}{\delta}} \cdot d} $.
Next, we consider the first term in \eqref{equ:distance two terms to bound Sigma}. Since $x^*_{A,\lambda}$ is the optimum of $h$ within $\operatorname{Span}(A)$, the gradient $\nabla h(x^*_{A,\lambda})$ is in the orthogonal subspace, i.e., $A^\top \nabla h(x^*_{A,\lambda}) = 0 $, thus $\Bar\Sigma  \nabla h(x^*_{A,\lambda}) = 0  $. The first term in \eqref{equ:distance two terms to bound Sigma} can be written as
\begin{equation*}
    \begin{aligned}
         I_1^2 &= \dabs{ \rbra{\mu_{k-1} - x^\star_{A,\lambda}} + \eta \Bar\Sigma \rbra{\nabla h(\mu_{k-1})-\nabla h(x^*_{A,\lambda})}}_{\Bar\Sigma^{-1}}^2 \\
        &=  \dabs{\mu_{k-1}  - x^\star_{A,\lambda} }_{{\Bar\Sigma}^{-1}}^2+ \eta^2 \dabs{ \nabla h(\mu_{k-1})-\nabla h(x^*_{A,\lambda})}_{\Bar\Sigma}^2 \\
        &+ 2 \pbra{\mu_{k-1} - x^\star_{A,\lambda},   \eta \rbra{\nabla h(\mu_{k-1})-\nabla h(x^*_{A,\lambda})} }. 
    \end{aligned}
\end{equation*}
Recall $h$ is $f$ adding a $\dabs{\cdot}_{\Bar\Sigma^{-1}}$ regularized term. We get $h$ is $(L+ \lambda )$-smooth with respect to semi norm $\dabs{\cdot}_{\Bar\Sigma^{-1}}$ which is derived from $f$ $L$-smooth. Also, $h$ is $\lambda  $-strongly concave with respect to semi norm $\dabs{\cdot}_{\Bar\Sigma^{-1}}$ since $f$ is concave. According to Lemma~\ref{lmm:opt}, we derive
\begin{equation*}
    \begin{aligned}
         \pbra{\mu_{k-1} - x^\star_{A,\lambda},   \nabla h(\mu_{k-1})-\nabla h(x^*_{A,\lambda})} \leq - \fs{\lambda (L + \lambda)}{L+ 2\lambda } \dabs{\mu_{k-1} - x^*_{A,\lambda}}^2_{\Bar\Sigma^{-1}}  - \fs{1}{L+ 2\lambda } \dabs{\nabla h(\mu_{k-1})-\nabla h(x^*_{A,\lambda})}^2_{\Bar\Sigma}.
    \end{aligned}
\end{equation*}
Plugin the formula of $I_1$, we get
\begin{equation*}
    \begin{aligned}
        I_1^2 &\leq \rbra{1 - \fs{2 \eta \lambda (L + \lambda)}{L+ 2\lambda } } \dabs{\mu_{k-1}  - x^\star_{A,\lambda} }_{{\Bar\Sigma}^{-1}}^2+ \rbra{\eta^2 - \fs{2 \eta }{L+ 2\lambda } } \dabs{ \nabla h(\mu_{k-1})-\nabla h(x^*_{A,\lambda})}_{\Bar\Sigma}^2.
    \end{aligned}
\end{equation*}
Since $\eta = 1 / \lambda $, it holds $\eta^2 - \fs{2 \eta }{L+ 2\lambda } > 0  $. Due to $h$ $(L + \lambda)$-smoothness, we get
\begin{equation*}
    \begin{aligned}
        I_1^2 &\leq \rbra{1 - \fs{2 \eta \lambda (L + \lambda)}{L+ 2\lambda } } \dabs{\mu_{k-1}  - x^\star_{A,\lambda} }_{{\Bar\Sigma}^{-1}}^2+ \rbra{\eta^2 - \fs{2 \eta }{L+ 2\lambda } }(L + \lambda)^2 \dabs{\mu_{k-1}  - x^\star_{A,\lambda} }_{{\Bar\Sigma}^{-1}}^2 \\
        &= \rbra{1 - \eta \rbra{L+\lambda}}^2 \dabs{\mu_{k-1}  - x^\star_{A,\lambda} }_{{\Bar\Sigma}^{-1}}^2, 
    \end{aligned}
\end{equation*}
thus, we get the bound of $I_1$ 
\begin{equation*}
    I_1 \leq \zeta \dabs{\mu_{k-1}  - x^\star_{A,\lambda} }_{{\Bar\Sigma}^{-1}},
\end{equation*}
where $\zeta := \abs{1 - \eta \rbra{L+\lambda}} $. 
Combing the upper bound of $I_1$ and $I_2$, we get with probability at least $1 - \delta/K$, for $1 < k \leq K$,
\begin{equation*}
     \dabs{\mu_k  - x^\star_{A,\lambda} }_{\Bar\Sigma^{-1}} \leq \zeta  \dabs{\mu_{k-1}  - x^\star_{A,\lambda} }_{\Bar\Sigma^{-1}} + \fs{M_0}{\sqrt{B_{k-1}}}.
\end{equation*}
As for $k=1$, by similar derivation, we can obtain $\dabs{\mu_{1}  - x^\star_{A,\lambda} }_{\Bar\Sigma^{-1}} \leq \zeta  \dabs{{z}_{0}  - x^\star_{A,\lambda} }_{\Bar\Sigma^{-1}}$.
By induction, we get with probability at least $1 - \rbra{(K-1)/K}\delta$, 
\begin{equation*}
    \begin{aligned}
        \dabs{\mu_K  - x^\star_{A,\lambda} }_{\Bar\Sigma^{-1}} \leq \zeta^K  \dabs{{z}_{0}  - x^\star_{A,\lambda} }_{\Bar\Sigma^{-1}}  + M_0 \sum_{k=1}^{K-1}\fs{\zeta^{K-k-1}}{\sqrt{B_k}}.
    \end{aligned}
\end{equation*}
Choose $B_k \geq \zeta^{-4k}(1-\zeta)^{-2} $ for all $k \in [K-1]$, then we can get 
\begin{equation}\label{equ:final iter distance Sigma}
    \dabs{\mu_K - x^\star_{A,\lambda} }_{\Bar\Sigma^{-1}} \leq \zeta^{K} \rbra{ \dabs{{z}_{0}  - x^\star_{A,\lambda} }_{\Bar\Sigma^{-1}}^2 + {M_1 \cdot\sqrt{d}   } }
\end{equation}
where $M_1 := \eta L 
\sqrt{2\log \rbra{\fs{2K}{\delta}}} $.
Since $h$ is $(L + \lambda)$-smooth with respect to $\dabs{\cdot}_{\Bar\Sigma^{-1}}$, it holds
\begin{equation*}
    \abs{h(\mu_K) - h(x^*_{A,\lambda}) - \pbra{\nabla h(x^*_{A,\lambda}), \mu_K - x^*_{A,\lambda}}} \leq \fs{L + \lambda }{2} \dabs{\mu_K - x^*_{A,\lambda}}^2_{\Bar\Sigma^{-1}}. 
\end{equation*}
Considering that $\nabla h(x^*_{A,\lambda}) \perp \operatorname{Span}(A) $ yields $\pbra{\nabla h(x^*_{A,\lambda}), \mu_K - x^*_{A,\lambda}} = 0$, we obtain the following by rearranging the equation above
\begin{align}
       f\rbra{x^*_{A,\lambda}} - f(\mu_K)  &\leq  \fs{\lambda}{2}\rbra{\dabs{x^*_{A,\lambda}-\bar{\mu}}^2_{\Bar\Sigma^{-1}} - \dabs{\mu_K -\bar{\mu} }^2_{\Bar\Sigma^{-1}}}  + \fs{L + \lambda}{2} \dabs{\mu_K - x^*_{A,\lambda}}^2_{\Bar\Sigma^{-1}} \label{equ:smooth rearrange Sigma} \\
       &\leq \sbra{\lambda \dabs{\bar{\mu} - x^*_{A,\lambda}}_{\Bar\Sigma^{-1}}\dabs{\mu_K - x^*_{A,\lambda}}_{\Bar\Sigma^{-1}} + (L + \lambda )\dabs{\mu_K - x^*_{A,\lambda}}^2_{\Bar\Sigma^{-1}}} . \nonumber
\end{align}
Substitute \eqref{equ:final iter distance Sigma} into above upper bound, with ${z}_0 = \bar{\mu}$ we have
\begin{equation*}
    \begin{aligned}
        f\rbra{x^*_{A,\lambda}} - f(\mu_K)  \lesssim  \zeta^{K} \cdot (L + \lambda  ) \sbra{ \dabs{\bar{\mu} - x^*_{A,\lambda}}_{\Bar\Sigma^{-1}}^2 + M_1^2 d}. 
    \end{aligned}
\end{equation*}
Since $ \dabs{\bar{\mu} - x^*_{A,\lambda}}_{\Bar\Sigma^{-1}}^2 = \dabs{A^\top \rbra{\bar{\mu} - x^*_{A,\lambda}}}_{\Bar{\Sigma}_u^{-1}}^2$ is the distance within $\operatorname{Span}(A)$, i.e., $\cO(d)$. Recall $\eta = 1/ \lambda$, $\zeta = \abs{1 - \eta \rbra{L+\lambda}} = {L /\lambda } $, $\lambda > L$, and $M_1 = \eta L \sqrt{2 \log \rbra{\fs{2K}{\delta}}} $. Therefore, we get the final result:
\begin{equation*}
    f\rbra{x^*_{A,\lambda}} - f(\mu_K)  \lesssim \lambda \rbra{ \fs{L}{\lambda}}^K d \log \rbra{\frac{K}{\delta}}, \quad \text{w.p.} 1 - \delta.
\end{equation*}

\end{proof}

\subsection{Proof of Theorem~\ref{thm:w. update w. subspace}} \label{sec:proof of opt theorem update}
We first provide a proof sketch.

\paragraph{Proof Sketch} The proof idea is similar to the proof of Theorem \ref{thm:subspace fully linear}.
For simplicity, we analyze the case where only the most recent sample batch $\cD_k$ is merged with $\cD_0$ for finetuning the score function. More specifically, we let $w_{k,i}=0$ for $0<i<k$ and $w_{k,0} = 1- w_{k,k}$. Similar to the proof of \cref{thm:fully linear},
 we obtain a recursive update rule given by
 \begin{equation}
 \label{equ:proj_gd}
    \mu_{k+1} =  \bar{z}_{k} + \eta_k \Bar{\Sigma} \sbra{ \nabla f(\bar{z}_{k}) - \rbra{1 - w_{k,k}}\eta_k^{-1} \cdot  \Bar{\Sigma}^{-1}  \rbra{\bar{z}_{k} - \bar{\mu}}  },
\end{equation}
where $\bar z_k\approx \mu_k$ is the empirical mean of previous samples.
This update rule also closely resembles the 
gradient ascent iteration for maximizing a regularized objective. A key difference here is that we can control the weights $w_{k,i}$ to reduce the impact of $\mathcal{D}_0$ and make the regularization term vanish to zero. Thus the mean $\mu_k$ eventually converges to the global maxima.

\begin{proof}
Define
\begin{equation*}
    \begin{aligned}
      \cH_0 &:= \sigma (\bar{\mu}), \\
      \cH_k &:= \sigma \rbra{\cH_{k-1}, \sigma \rbra{z_{k-1,1}, \ldots, z_{k-1,B_{k-1}}} }, \quad k \in [K-1], \\
      \mu_{k} &:= \EE [z_{k,i} \mid \cH_{k-1}], \quad k \in [K].
\end{aligned}
\end{equation*}
According to Lemma~\ref{lmm:general score match}, with freezing $C_t$ in class \eqref{equ:fully linear class}, the pre-trained score in Round $k$ is $s_{\theta_{k+1}}(x_t,t) = - \rbra{\alpha^2(t) \bar{\Sigma}  + h(t)I_D}\rbra{x_t - \alpha(t)\bar{x}_{k}  } $ where $\bar{x}_{k} = \sum_{j=0}^{k} w_{k,j} \bar{z}_j $ and $\bar{z}_j = {\EE}_{x \in \cD_j} [x] $. 
By choosing $y_{k}$, and weights $w_{k,j}$ as 
\begin{equation*}
    \begin{aligned}
        y_{k} &= \eta_{k} \cdot \rbra{\sigma^2 + g_{k}^\top A\bar{\Sigma}_{u}A^\top g_{k}} + g_{k}^\top AA^\top \bar{x}_{k},\\
        w_{k,0} &= 1 - w_{k} \\
        w_{k,j} &= 0 , \,\, 1 \leq j < k, \\
        w_{k,k} &= w_{k},
    \end{aligned}
\end{equation*}
where $\eta_{k} > 0$, $0<w_{k}<1$ will be specified later. And we choose $\beta(t)$ at Round $k$  as $\beta(t) = \fs{1}{2} \rbra{\sigma^2 + g_{k-1}^\top \Bar{\Sigma}^{-1}  \rbra{ I_D +  \alpha^2(t)\Bar{\Sigma}/h(t) }^{-1} g_{k-1}}^{-1}$. Lemma~\ref{lmm:posterior distribution general} gives the mean of distribution of $z_{k+1,i}$ as
\begin{equation}\label{equ:mean k round}
    \mu_{k+1} = \bar{x}_k + \eta_k \Bar{\Sigma} g_k,
\end{equation}
and the output distribution
\begin{equation}\label{equ:adapt_gen_gaussian}
    \cN \rbra{\bar{x}_{K-1} + \eta_{k-1} \bar{\Sigma} g_{K-1}, \Bar{\Sigma} -  \frac{\Bar{\Sigma} g_{K-1}   g_{K-1}^\top \Bar{\Sigma}}{\sigma^2 + g_{K-1}^\top \Bar{\Sigma} g_{K-1}} }. 
\end{equation}
Applying Lemma~\ref{lmm:posterior distribution general}  yields ${z}_{k,i}  \in \operatorname{Span(A)} $, thus, $ \bar{x}_{k} = AA^\top \bar{x}_{k} $ and $\bar{\Sigma} = A \bar{\Sigma}_{u} A^\top $, we get the update rule reduced to
\begin{equation*}
    \begin{aligned}
        \mu_{k+1} &= AA^\top \rbra{(1-w_{k})\bar{\mu} + w_{k}\bar{z}_{k}} + \eta_{k-1} A \bar{\Sigma}_{u} A^\top g_{k-1} \\
        &=  AA^\top\bar{z}_{k} + \eta_{k}A \bar{\Sigma}_{u} A^\top \rbra{\nabla f(\bar{z}_{k}) - \eta_{k}^{-1}(1-w_{k}) A \bar{\Sigma}_{u}^{-1} A^\top\rbra{\bar{z}_{k} - \bar{\mu}} }.
    \end{aligned}
\end{equation*}
We set $w_{k} = 1- \eta_{k} \lambda$ and set $\eta_{k} = \eta$, where $\lambda, \eta > 0$ will be specified later. Therefore, we have
\begin{equation}\label{equ:shrink distance}
    \begin{aligned}
        \mu_{k+1} =  AA^\top\bar{z}_{k} + \eta A \bar{\Sigma}_{u} A^\top \nabla h_{\lambda}(\bar{z}_{k}),
    \end{aligned}
\end{equation}
where $h_{\lambda}(x):= f(x) - (\lambda/2) \dabs{x - \bar{\mu}}^2_{\Bar\Sigma^{-1}} $.
Define $ x^\star_{A,\lambda} = \argmax_{x=Au} h_{\lambda}(x) $. With some similar steps in proof in Appendix~\ref{pf:subspace fully linear}, by choosing $B_k \geq \zeta^{-4k}(1-\zeta)^{-2}$, together with ${z}_0 = \bar{\mu}$, we get
\begin{equation*}
    \dabs{\mu_{K} -  x^\star_{A,\lambda}}_{\Bar\Sigma^{-1}} \lesssim \zeta^{K} \rbra{\dabs{\bar \mu  - x^\star_{A,\lambda} }_{\Bar\Sigma^{-1}} +  M_1 \cdot 
    \sqrt{d}    } , \quad \text{w.p.} \,\,  1- \delta,
\end{equation*}
with $\eta = \fs{2 }{L+ 2\lambda } $,  $\zeta = \abs{1 - \eta (L+\lambda)} $ and $ M_1 = 2L \sqrt{(1+\eta^2) \log \rbra{\fs{2K}{\delta}}}$. Also, we can get \eqref{equ:smooth rearrange Sigma} as in proof in Appendix~\ref{pf:subspace fully linear}. We restate it here:
\begin{equation}\label{equ:restate rearrange smooth}
    f\rbra{x^*_{A,\lambda}} - f(\tilde z_K)  \leq  \fs{\lambda}{2}\rbra{\dabs{x^*_{A,\lambda}-\bar{\mu}}^2_{\Bar\Sigma^{-1}} - \dabs{\mu_{K} -\bar{\mu} }^2_{\Bar\Sigma^{-1}}}  + \fs{L + \lambda}{2} \dabs{\mu_{K} - x^*_{A,\lambda}}^2_{\Bar\Sigma^{-1}}.
\end{equation}
Since $f$ is concave,
\begin{equation*}
    \begin{aligned}
       f\rbra{x_{A}^\ast} -  f\rbra{x^*_{A,\lambda}} \leq \pbra{\nabla  f\rbra{x^*_{A,\lambda}}, x_{A}^\ast-x^*_{A,\lambda}} = \lambda \pbra{\Bar\Sigma^{-1}(x^*_{A,\lambda}-\bar{\mu}), x_{A}^\ast-x^*_{A,\lambda}}.
    \end{aligned}
\end{equation*}
Adding \eqref{equ:restate rearrange smooth}, it holds
\begin{equation*}
    \begin{aligned}
        f\rbra{x^*_{A,\lambda}} - f(\mu_K) &\leq  \fs{\lambda}{2}\rbra{\dabs{x^*_{A,\lambda}-\bar{\mu}}^2_{\Bar\Sigma^{-1}} - \dabs{x^*_{A,\lambda}-x_{A}^\ast}^2_{\Bar\Sigma^{-1}} -  \dabs{\mu_{K} -\bar{\mu}}^2_{\Bar\Sigma^{-1}} }  + \fs{L + \lambda }{2} \dabs{\mu_{K} - x^*_{A,\lambda}}^2_{\Bar\Sigma^{-1}} .
    \end{aligned}
\end{equation*}
Due to \eqref{equ:shrink distance}, we have, it holds w.p. $1-\delta$,
\begin{equation*}
    f\rbra{x^*_{A,\lambda}} - f(\mu_K) \lesssim  \sbra{\lambda \dabs{x_{A}^\ast - \bar{\mu}}^2_{\Bar\Sigma^{-1}} + (L+\lambda) \zeta^K \rbra{\dabs{\Bar{\mu} - x^*_{A,\lambda}}^2_{\Bar\Sigma^{-1}} +  M_1 d} }.
\end{equation*}
We choose $\lambda = L \log K /(4K)$ and get
\begin{equation*}
\begin{aligned}
    f\rbra{x^*_{A,\lambda}} - f(\mu_K) &\lesssim \fs{L\log K}{K} \cdot \sbra{ \dabs{x_{A}^\ast - \bar{\mu}}^2_{\Bar\Sigma^{-1}} + \dabs{\Bar{\mu} - x^*_{A,\lambda}}^2_{\Bar\Sigma^{-1}} +  M_1 d}, \,\, \text{w.p.} \,\,  1- \delta.
\end{aligned}
\end{equation*}
With assuming  $\dabs{x^\star_{A,\lambda}}$ is bounded, we derive
\begin{equation*}
    f\rbra{x^*_{A,\lambda}} - f(\mu_K) = \cO \rbra{\fs{dL^2\log K}{K} \cdot \log \rbra{\frac{K}{\delta}}}, \,\, \text{w.p.} \,\,  1- \delta.
\end{equation*}

\subsection{Auxiliary Lemma}
The following is a standard result in convex optimization utilized in previous proofs.
\begin{lemma} \label{lmm:opt}
    Let $f$ be $\alpha$-strongly concave and $\beta$-smooth with respect to the (semi) norm $\dabs{\cdot}_{\Sigma^{-1}}$, for all $x$ and $y$, it holds
    \begin{equation}
        - \pbra{\nabla f(x) - \nabla f(y), x-y} \geq \fs{\alpha \beta}{\alpha+\beta} \dabs{x - y}^2_{\Sigma^{-1}} + \fs{1}{\alpha + \beta} \dabs{\nabla f(x) - \nabla f(y)}^2_{\Sigma}.
    \end{equation}
\end{lemma}

\begin{proof}
    See \citet[Lemma 3.11]{bubeck2015convex} for a proof.
\end{proof}

\end{proof}
\section{Additional Materials for Experiments}
\stopcontents 
\label{sec:experiment details}

\subsection{Simulation}\label{sec:simulation_detail}
We experiment with our design of the gradient guidance as well as \cref{alg:main} and \cref{alg:main_w_update}. Going beyond our theoretical assumptions, we adopt a 15M-parameter U-Net as the score function class for training and fine-tuning our diffusion model. 

\subsubsection{Experiment Setup}
\label{training_details}

For linear data structure, we set the data's ambient dimension as $D = 64$ and the linear subspace dimension as $d = 16$. The linear subspace is represented by an orthogonal matrix $A \in \RR^{D \times d}$. We randomly generate a matrix $A$ and fix it once generated. After that, we sample a data point  $X$ by first randomly sampling a latent variable $U \sim \cN(0, I_d)$ and computing $X = AU$. We independently sample a total of $65536$ data points as our pre-training data set. For nonlinear data structure, data are uniformly sampled from a unit ball in $\mathbb{R}^{64}$.

 The objective functions considered in our experiments are
$f_1(x) = 10 - (\theta^\top x - 3)^2$ and $f_2(x) = 5 - 0.5 \lVert x - b \rVert.$
Here, $\theta$ and $b$ are randomly generated and fixed afterward. 
Since our data assumes a low-dimensional subspace representation, it is convenient to decompose $\theta$ into $\theta_{\perp} = (I - AA^\top) \theta$ and $\theta_{\parallel} = AA^\top \theta$, representing the off-support and on-support components. We refer to $\frac{\lVert \theta_{\perp}\rVert}{\lVert \theta_{\parallel} \rVert}$ as the off/on-support ratio. Analogously, for a generated sample, we can also define its off/on-support ratio. Clearly, a small off/on-support ratio indicates close vicinity to the subspace.

\paragraph{Score Network Pre-training.} We utilize a version of the U-Net \citep{ronneberger2015u}, with 14.8M trainable parameters. Note that this is a complicated network going beyond the linear score function class considered in our theories. Following the implementation of Denoising Diffusion Probabilistic Models (DDPM, \citet{ho2020denoising}), we train the U-Net o estimate the score function $\nabla \log p_t$, via minimizing the score matching loss introduced in Eqn.~\eqref{equ:score match objective}. We discretize the backward process to have $200$ time steps as in \citet{nichol2021improved}, and the U-Net is trained using our generated data set for 20 epochs. We use Adam as the optimizer, set the batch size as $32$, and set the learning rate to be $10^{-4}$. After the pre-training phase, we confirmed that the data subspace structure is well learned, as the generated samples using the pre-trained diffusion model have an average off/on-support ratio of $0.039$.

\paragraph{Implementation of \cref{alg:main}.}
In each iteration of \cref{alg:main}, we need to compute the gradient guidance $\texttt{G}_{loss}$. We set the targeted $y$ value at the $k$-th iteration as $y_k = \delta + g_k^\top z_k$, where $\delta_k$ specifies the increment per iteration. The choice on $\delta_k$ is instance-dependent and we set it via tuning for near-optimal in different experiments. For comparing naive gradient with gradient guidance in Figure~\ref{fig:cmp}, we set $\delta = 0.2$ and $0.9$, respectively for using naive gradient $\texttt{G}$ and gradient guidance $\texttt{G}_{loss}$. In Figure~\ref{fig:regularized}, we choose $\delta$ to be (a) $0.05$, (b) $0.2$, (c) $1$, and (d) $1$, corresponding to each panel. We initialize Algorithm~\ref{alg:main} with a batch of $32$ samples generated by the pre-trained model. Each sample determines an optimization trajectory. We repeat Algorithm~\ref{alg:main} for $5$ times with different random seeds and report the error bars. 

\paragraph{Implementation of \cref{alg:main_w_update}.}
Algorithm~\ref{alg:main_w_update} differs from Algorithm~\ref{alg:main} in that it allows additional fine-tuning of the pre-trained score network. In practice, to update the score network incorporating newly generated data, one does not have to exactly solve \eqref{equ:score update objective} by re-training the full model from scratch. Instead, \eqref{equ:score update objective} can be viewed as a guideline that motivates more computationally efficient ways for updating the pre-trained score. It is a common practice to only \textbf{fine-tune} the weights of the old model by performing gradient descent over a few batches of newly generated data, which is similar to the spirit of \eqref{equ:score update objective}. To be more specific, we adopt a computationally lightweight fine-tuning strategy: We only perform one Adam optimization step using the re-weighted loss given by Eqn.~\eqref{equ:score update objective} with a batch of $32$ generated samples. We set the learning rate as $10^{-6}$. This simple strategy already demonstrates good performances as shown in Figure~\ref{fig:alg2}. Other implementation details are kept the same as those of Algorithm~\ref{alg:main}.


\subsubsection{Results}
We first demonstrate our gradient guidance $\texttt{G}_{loss}$ preserves the subspace structure learned from the pre-trained model. 
For comparison, we also tested the naive  guidance $\texttt{G}$ defined following \cref{prop:g=scaler_grad} (with $\Sigma=I$). For a quick reference, we repeat the definition here:
\begin{equation*}
        \texttt{G}(x_t, t) :=  \beta(t) \rbra{y - g^\top \EE[x_0|x_t]}g,
\end{equation*}
where $\beta(t) > 0$ and $y \in \RR$ are tuning parameters, and $\EE[x_0|x_t]$ is the conditional expectation of $x_0$ given noise corrupted data $x_t$. For implementation, we replace $\EE[x_0 | x_t]$ by its look-ahead estimator $\hat{\EE}[x_0 | x_t]$ based on the Tweedie's formular.

\paragraph{Comparing \texttt{G} and $\texttt{G}_{loss}$ on Preserving Subspace Structure.}  \cref{fig:cmp} (a), (c) verify that the naive gradient $\texttt{G}$ performs much worse than $\texttt{G}_{loss}$ in preserving the linear subspace structure. It is consistent with our theoretical finding that the gradient guidance $\texttt{G}_{loss}$ keeps the generated sample close to the latent subspace, with substantially smaller off-support errors. When allowing adaptive score fine-tuning in Algorithm~\ref{alg:main_w_update}, \cref{fig:cmp} (b), (d) show that the off-support error increases as the model gets fine-tuned using self-generated data, due to increasing distribution shift. Even in this case, the naive gradient $\texttt{G}$ leads to much more severe off-support errors as compared to $\texttt{G}_{loss}$.

\begin{figure}[!htb]
\centering

\subfigure[Algorithm~\ref{alg:main}]{\includegraphics[width=0.24\textwidth]{figures/figure_2/v1_v2_theta_9_ratio_100_v5.png}}
\subfigure[Algorithm~\ref{alg:main_w_update}]{\includegraphics[width=0.24\textwidth]{figures/figure_2/rebuttal_v1_v2_theta_9_update_score_revised_ratio_100_v5.png}}
 \subfigure[300-350 round of (a)]{\includegraphics[width=0.24\textwidth]{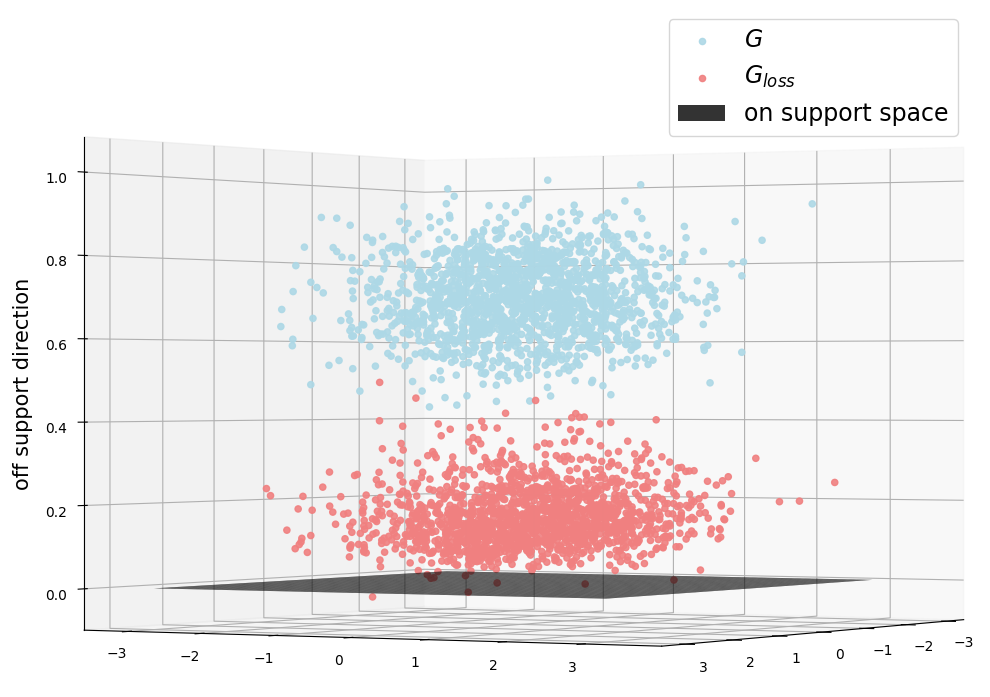}}
  \subfigure[1000-1200 round of (b)]{\includegraphics[width=0.24\textwidth]{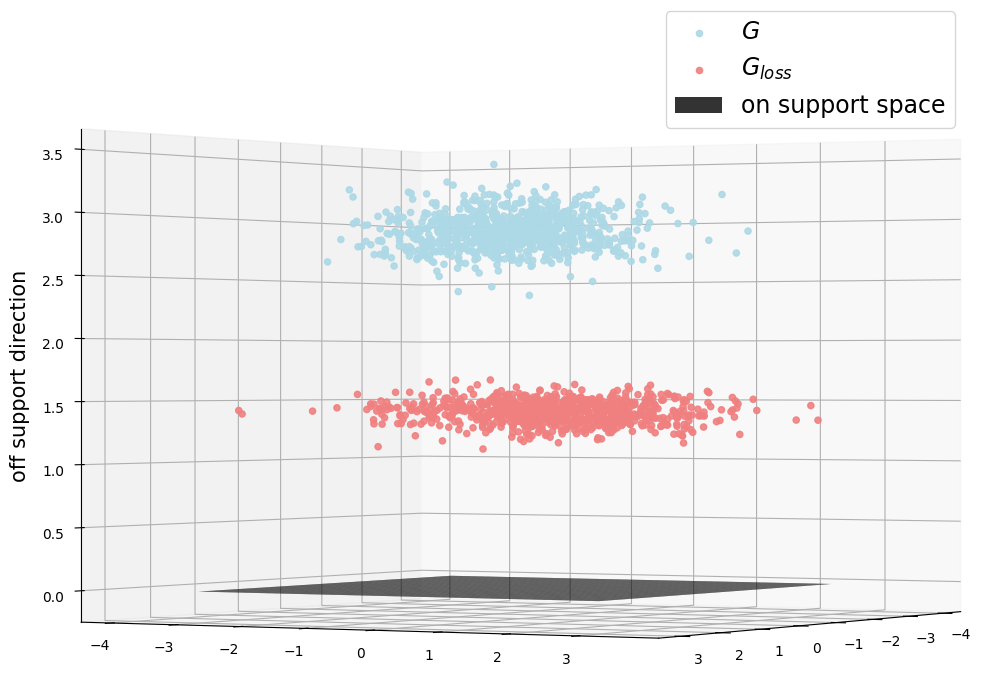}}

\caption
{\small{
\textbf{Comparison between two types of gradient guidance $\texttt{G}$ and $\texttt{G}_{loss}$.} We plot the off/on support ratio of the generated samples, denoted by $r_\mathrm{off} =\frac{\Vert{x_\bot}\Vert}{\Vert{x_\Vert}\Vert}$. The objective function is $f_1(x)$, with $\theta$ having an off/on-support ratio of $9$. 
}}
\label{fig:cmp}
\end{figure}

\paragraph{\cref {alg:main} Converges to Regularized Optima.} 
We plot the convergence of Algorithm~\ref{alg:main} in terms of the objective value in \cref{fig:regularized}. Figure~\ref{fig:regularized} (a),(b) are for the objective function $f_1= 10 - (\theta^\top x - 3)^2$ as the objective function, while Figure~\ref{fig:regularized}(c),(d) are for the objective $f_2= 5 - 0.5 \lVert x - b \rVert$. We observe that the algorithm converges to reach some sub-optimal objective value, but there remains a gap to the maximal value. This is consistent with our theory that the pre-trained model essentially acts as a regularization in addition to the objective function. Adding gradient guidance alone cannot reach global maxima. This coincides with our theoretical findings in Theorem~\ref{thm:subspace fully linear}.

\begin{figure}[!htb] 
\centering
\subfigure[$\theta=A\beta^*$]{\includegraphics[width=0.24\textwidth]{figures/fig_alg1/v1_v2_theta_A_reward_100_None_single_v2.png}}
\subfigure[$\frac{\Vert{\theta_\bot}\Vert}{\Vert{\theta_\Vert}\Vert}=9$]{\includegraphics[width=0.24\textwidth]{figures/fig_alg1/v1_v2_theta_9_reward_100_1001_single_new_v2.png}}
\subfigure[$b =4 \cdot \textbf{1}_{D}$]{\includegraphics[width=0.24\textwidth]{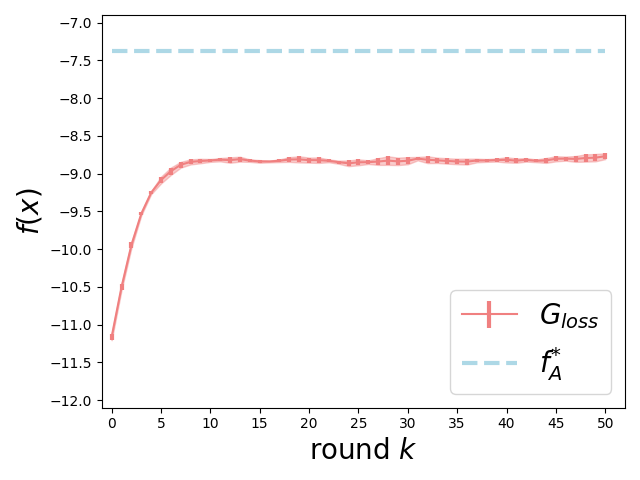}}
\subfigure[$b\sim\cN(4 \cdot \textbf{1}_{D}, 9 \cdot I_D)$]{\includegraphics[width=0.24\textwidth]{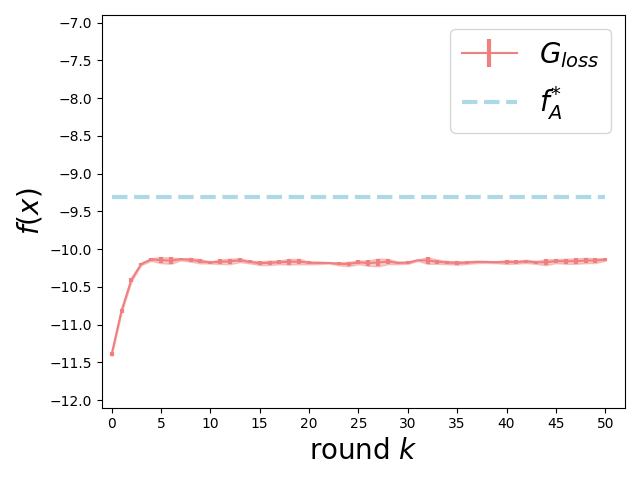}}
\caption
{\small{
\textbf{Convergence of Algorithm~\ref{alg:main} under different objectives}. Objectives are $f_1(x)$ for (a) and (b), and $f_2(x)$ for (c) and (d). Parameters $\theta$ and $b$ are specified as (a) $\theta=A\beta^*$ with $\beta^*$ being sampled from the unit ball in $\RR^d$; (b) the off/on-support ratio of $\theta$ being $9$ (same as Figure~\ref{fig:cmp}); 
(c) and (d) choosing $b$ as a homogeneous vector or randomly from a Gaussian distribution. All the experiments adopt the gradient guidance $\texttt{G}_{loss}$.}}
\label{fig:regularized}
\end{figure}

\paragraph{\cref {alg:main_w_update} Converges to Global Optima.} Algorithm~\ref{alg:main_w_update} converges to the maximal value of the objective function $f_1= 10 - (\theta^\top x - 3)^2 $ as shown in \cref{fig:alg2}(a). In \cref{fig:alg2}(b), we visualize the distribution of generated samples of \cref{alg:main} (blue) and \ref{alg:main_w_update} (red), respectively, as the iteration evolves. We see that samples from \cref{alg:main} mostly stay close to the pre-training data distribution (area described by the dotted contour). In constrast, samples of  \cref{alg:main_w_update} move outside the contour, as the diffusion model gets fine-tuned using self-generated data.

\begin{figure}[!htb]
\centering
\subfigure[Convergence of Algorithm~\ref{alg:main_w_update}]{\includegraphics[width=0.32\textwidth]{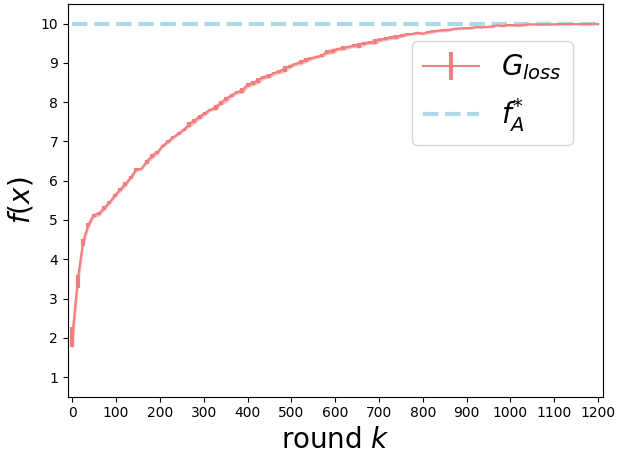}}
\subfigure[Distribution of generated samples]{\includegraphics[width=0.48\textwidth]{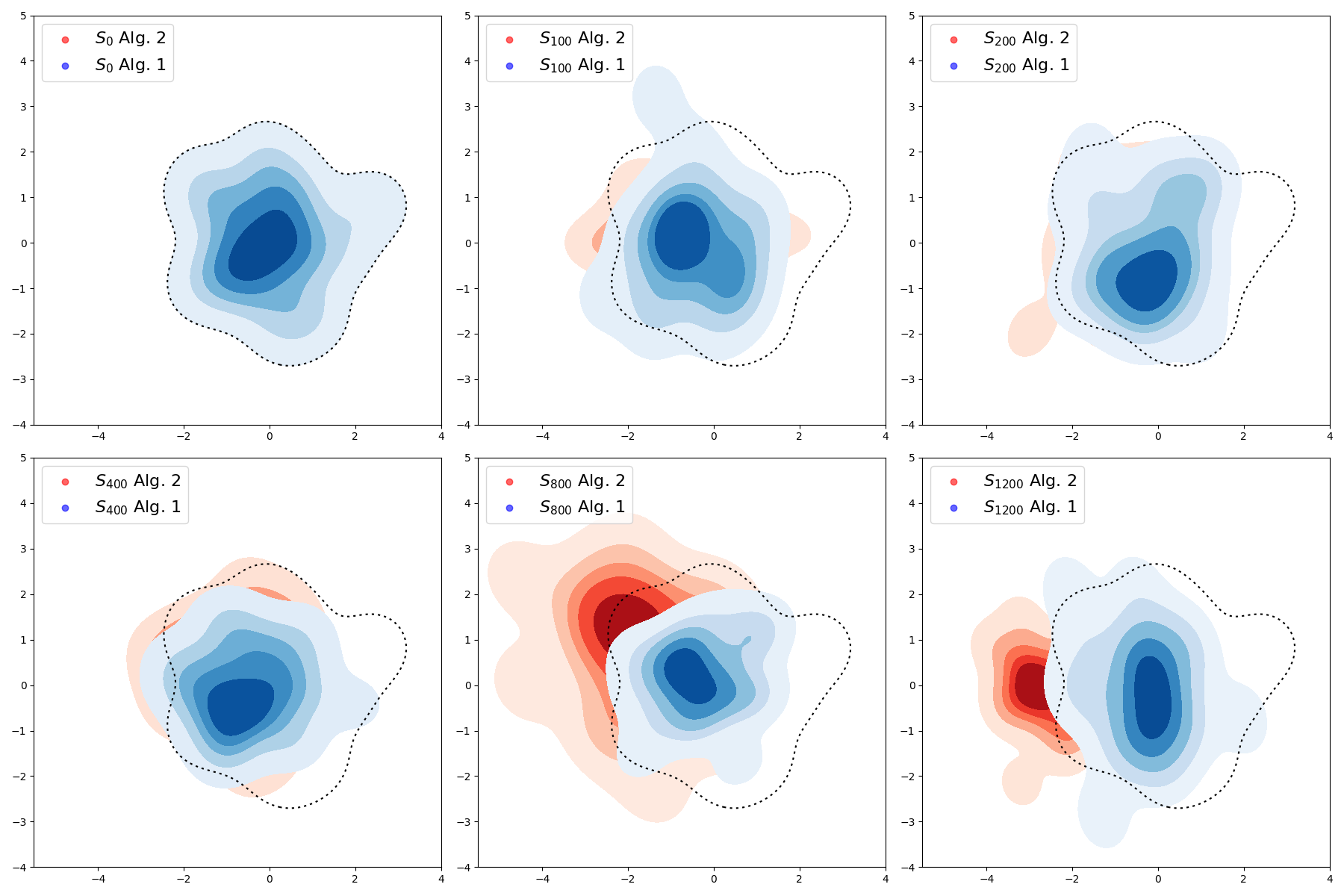}
}
\caption
{\small{
\textbf{Convergence of Algorithm~\ref{alg:main_w_update}}. Panel (a) plots the objective values achieved by Algorithm~\ref{alg:main_w_update} as a function of iterations. Here $\theta$ is chosen the same as in Figure~\ref{fig:regularized} (b) with off/on-support ratio $\frac{\Vert{\theta_\bot}\Vert}{\Vert{\theta_\Vert}\Vert}=9$. Panel (b) visualizes the distribution of the generated samples of Algorithm~\ref{alg:main_w_update} (red) across the iterations. For comparison, we also visualize the distribution of generated samples of Algorithm~\ref{alg:main} (blue).
}} 
\label{fig:alg2}
\end{figure}

\paragraph{Results for Nonlinear Data Structure.} We apply \cref{alg:main} to data uniformly sampled from a unit ball in $\mathbb{R}^{64}$. The objective reward function is defined as $f(x) = \theta^\top x$, where $\left \lVert \theta \right \rVert = 1$.
The left panel of  \cref{fig:sim_nonlinear} demonstrates that rewards increase and converge when using  \cref{alg:main}. Higher guidance strength $\Delta$ (corresponding to lower regularization) results in a higher convergent reward.
The right panel of \cref{fig:sim_nonlinear} shows that, for the same reward level, gradient guidance achieves a smaller deviation from the unit ball compared to the naive gradient approach. This suggests that gradient guidance can better preserve data structure for nonlinear manifolds.

\begin{figure}[h]
    \centering
\subfigure{\includegraphics[width=0.40\textwidth]{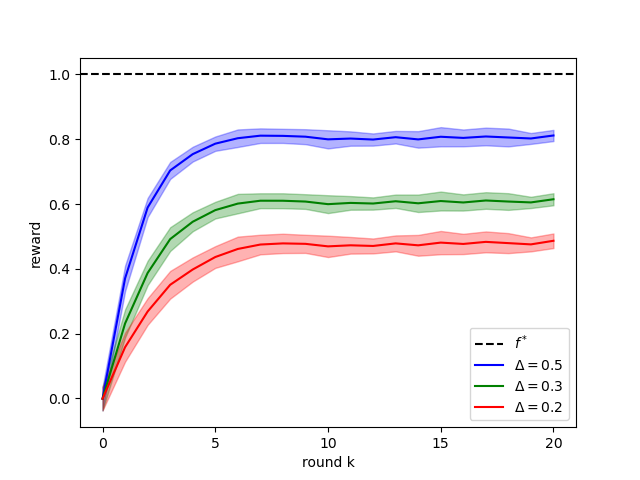} }
\subfigure{\includegraphics[width=0.40\textwidth]{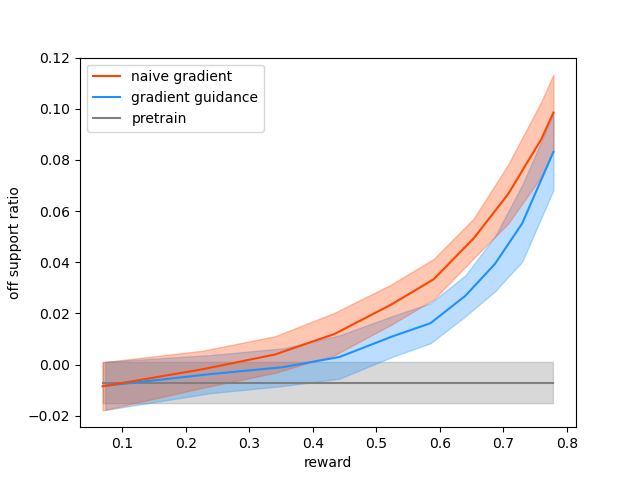}}
    \caption{\textbf{Nonlinear data structure experiment.} We apply \cref{alg:main} to data uniformly sampled from a unit ball in $\mathbb{R}^{64}$. The objective reward function is $f(x) = \theta^\top x$, where $\left \lVert \theta \right \rVert = 1$. \textbf{Left:} Rewards increase and converge with \cref{alg:main}. Higher guidance strength $\Delta$ (lower regularization) results in a higher convergent reward. \textbf{Right:} For the same reward, gradient guidance achieves a smaller deviation from the unit ball compared to the naive gradient. This indicates that gradient guidance can preserve data structure for nonlinear manifolds.}
    \label{fig:sim_nonlinear}
\end{figure}

\subsection{Image Generation}\label{sec:image_detail}

\paragraph{Hyperparameters.} Since tuning parameters $y$ and $\beta(t)$ are both for the strength of guidance,  we can fix one of them. We choose $\beta(t) = 100$ as suggested by \cite{yuan2023reward}, and set a series of $y \in \cbra{2,4,10}$. We run batch size $=  20$ samples parallelly across iterations to evaluate the rewards. The prompt is uniformly sampled from the $1000$ classes of ImageNet \citep{deng2009imagenet}.

\subsection{Time Efficiency}
We summarize the time cost of our experiments on one NVIDIA A100 GPU in \cref{table_time}.
\begin{table}[H]
    \centering
    \begin{tabular}{cc|ccc}
    \toprule
        &\multicolumn{2}{c}{Total runtime (iterations)} & Per iteration & No guidance \\ 
        \midrule  
       Simulation  & {\color{red}3.8 min} (50 iter) & 76 min (1000 iter) & 4.6 s & 2.6 s\\
       Image & \multicolumn{2}{c}{{\color{red}1.3 min} (5 iter)}  & 15.8 s & 4.9 s\\
       \bottomrule
    \end{tabular}
    \vspace{5pt}
    \caption{Runtime Efficiency of Algorithm~1. {\color{red} Red} refers to the total time to converge. No guidance refers to the time for one-time inference of the pre-trained model.}
    \label{table_time}
\end{table}


\newpage
\begin{thebibliography}{68}
\providecommand{\natexlab}[1]{#1}
\providecommand{\url}[1]{\texttt{#1}}
\expandafter\ifx\csname urlstyle\endcsname\relax
  \providecommand{\doi}[1]{doi: #1}\else
  \providecommand{\doi}{doi: \begingroup \urlstyle{rm}\Url}\fi

\bibitem[Ajay et~al.(2022)Ajay, Du, Gupta, Tenenbaum, Jaakkola, and Agrawal]{ajay2022conditional}
Anurag Ajay, Yilun Du, Abhi Gupta, Joshua Tenenbaum, Tommi Jaakkola, and Pulkit Agrawal.
\newblock Is conditional generative modeling all you need for decision-making?
\newblock \emph{arXiv preprint arXiv:2211.15657}, 2022.

\bibitem[Albergo et~al.(2023)Albergo, Boffi, and Vanden-Eijnden]{albergo2023stochastic}
Michael~S Albergo, Nicholas~M Boffi, and Eric Vanden-Eijnden.
\newblock Stochastic interpolants: A unifying framework for flows and diffusions.
\newblock \emph{arXiv preprint arXiv:2303.08797}, 2023.

\bibitem[Anderson(1982)]{anderson1982reverse}
Brian~DO Anderson.
\newblock Reverse-time diffusion equation models.
\newblock \emph{Stochastic Processes and their Applications}, 12\penalty0 (3):\penalty0 313--326, 1982.

\bibitem[Bansal et~al.(2023)Bansal, Chu, Schwarzschild, Sengupta, Goldblum, Geiping, and Goldstein]{bansal2023universal}
Arpit Bansal, Hong-Min Chu, Avi Schwarzschild, Soumyadip Sengupta, Micah Goldblum, Jonas Geiping, and Tom Goldstein.
\newblock Universal guidance for diffusion models.
\newblock In \emph{Proceedings of the IEEE/CVF Conference on Computer Vision and Pattern Recognition}, pages 843--852, 2023.

\bibitem[Ben-Hamu et~al.(2024)Ben-Hamu, Puny, Gat, Karrer, Singer, and Lipman]{ben2024d}
Heli Ben-Hamu, Omri Puny, Itai Gat, Brian Karrer, Uriel Singer, and Yaron Lipman.
\newblock D-flow: Differentiating through flows for controlled generation.
\newblock \emph{arXiv preprint arXiv:2402.14017}, 2024.

\bibitem[Benton et~al.(2023)Benton, De~Bortoli, Doucet, and Deligiannidis]{benton2023linear}
Joe Benton, Valentin De~Bortoli, Arnaud Doucet, and George Deligiannidis.
\newblock Linear convergence bounds for diffusion models via stochastic localization.
\newblock \emph{arXiv preprint arXiv:2308.03686}, 2023.

\bibitem[Black et~al.(2023)Black, Janner, Du, Kostrikov, and Levine]{black2023training}
Kevin Black, Michael Janner, Yilun Du, Ilya Kostrikov, and Sergey Levine.
\newblock Training diffusion models with reinforcement learning.
\newblock \emph{arXiv preprint arXiv:2305.13301}, 2023.

\bibitem[Block et~al.(2020)Block, Mroueh, and Rakhlin]{block2020generative}
Adam Block, Youssef Mroueh, and Alexander Rakhlin.
\newblock Generative modeling with denoising auto-encoders and langevin sampling.
\newblock \emph{arXiv preprint arXiv:2002.00107}, 2020.

\bibitem[Bubeck et~al.(2015)]{bubeck2015convex}
S{\'e}bastien Bubeck et~al.
\newblock Convex optimization: Algorithms and complexity.
\newblock \emph{Foundations and Trends{\textregistered} in Machine Learning}, 8\penalty0 (3-4):\penalty0 231--357, 2015.

\bibitem[Chen et~al.(2023{\natexlab{a}})Chen, Huang, Zhao, and Wang]{chen2023score}
Minshuo Chen, Kaixuan Huang, Tuo Zhao, and Mengdi Wang.
\newblock Score approximation, estimation and distribution recovery of diffusion models on low-dimensional data.
\newblock \emph{arXiv preprint arXiv:2302.07194}, 2023{\natexlab{a}}.

\bibitem[Chen et~al.(2024)Chen, Mei, Fan, and Wang]{chen2024overview}
Minshuo Chen, Song Mei, Jianqing Fan, and Mengdi Wang.
\newblock An overview of diffusion models: Applications, guided generation, statistical rates and optimization.
\newblock \emph{arXiv preprint arXiv:2404.07771}, 2024.

\bibitem[Chen et~al.(2022)Chen, Chewi, Li, Li, Salim, and Zhang]{chen2022sampling}
Sitan Chen, Sinho Chewi, Jerry Li, Yuanzhi Li, Adil Salim, and Anru~R Zhang.
\newblock Sampling is as easy as learning the score: theory for diffusion models with minimal data assumptions.
\newblock \emph{arXiv preprint arXiv:2209.11215}, 2022.

\bibitem[Chen et~al.(2023{\natexlab{b}})Chen, Chewi, Lee, Li, Lu, and Salim]{chen2023probability}
Sitan Chen, Sinho Chewi, Holden Lee, Yuanzhi Li, Jianfeng Lu, and Adil Salim.
\newblock The probability flow ode is provably fast.
\newblock \emph{arXiv preprint arXiv:2305.11798}, 2023{\natexlab{b}}.

\bibitem[Chen et~al.(2023{\natexlab{c}})Chen, Daras, and Dimakis]{chen2023restoration}
Sitan Chen, Giannis Daras, and Alex Dimakis.
\newblock Restoration-degradation beyond linear diffusions: A non-asymptotic analysis for ddim-type samplers.
\newblock In \emph{International Conference on Machine Learning}, pages 4462--4484. PMLR, 2023{\natexlab{c}}.

\bibitem[Chung et~al.(2022{\natexlab{a}})Chung, Kim, Mccann, Klasky, and Ye]{chung2022diffusion}
Hyungjin Chung, Jeongsol Kim, Michael~T Mccann, Marc~L Klasky, and Jong~Chul Ye.
\newblock Diffusion posterior sampling for general noisy inverse problems.
\newblock \emph{arXiv preprint arXiv:2209.14687}, 2022{\natexlab{a}}.

\bibitem[Chung et~al.(2022{\natexlab{b}})Chung, Sim, Ryu, and Ye]{chung2022improving}
Hyungjin Chung, Byeongsu Sim, Dohoon Ryu, and Jong~Chul Ye.
\newblock Improving diffusion models for inverse problems using manifold constraints.
\newblock \emph{Advances in Neural Information Processing Systems}, 35:\penalty0 25683--25696, 2022{\natexlab{b}}.

\bibitem[Clark et~al.(2023)Clark, Vicol, Swersky, and Fleet]{clark2023directly}
Kevin Clark, Paul Vicol, Kevin Swersky, and David~J Fleet.
\newblock Directly fine-tuning diffusion models on differentiable rewards.
\newblock \emph{arXiv preprint arXiv:2309.17400}, 2023.

\bibitem[De~Bortoli(2022)]{de2022convergence}
Valentin De~Bortoli.
\newblock Convergence of denoising diffusion models under the manifold hypothesis.
\newblock \emph{arXiv preprint arXiv:2208.05314}, 2022.

\bibitem[De~Bortoli et~al.(2021)De~Bortoli, Thornton, Heng, and Doucet]{de2021diffusion}
Valentin De~Bortoli, James Thornton, Jeremy Heng, and Arnaud Doucet.
\newblock Diffusion schr{\"o}dinger bridge with applications to score-based generative modeling.
\newblock \emph{Advances in Neural Information Processing Systems}, 34:\penalty0 17695--17709, 2021.

\bibitem[Deng et~al.(2009)Deng, Dong, Socher, Li, Li, and Fei-Fei]{deng2009imagenet}
Jia Deng, Wei Dong, Richard Socher, Li-Jia Li, Kai Li, and Li~Fei-Fei.
\newblock Imagenet: A large-scale hierarchical image database.
\newblock In \emph{2009 IEEE conference on computer vision and pattern recognition}, pages 248--255. Ieee, 2009.

\bibitem[Dhariwal and Nichol(2021)]{dhariwal2021diffusion}
Prafulla Dhariwal and Alexander Nichol.
\newblock Diffusion models beat gans on image synthesis.
\newblock \emph{Advances in neural information processing systems}, 34:\penalty0 8780--8794, 2021.

\bibitem[Efron(2011)]{efron2011tweedie}
Bradley Efron.
\newblock Tweedie’s formula and selection bias.
\newblock \emph{Journal of the American Statistical Association}, 106\penalty0 (496):\penalty0 1602--1614, 2011.

\bibitem[El~Alaoui et~al.(2023)El~Alaoui, Montanari, and Sellke]{el2023sampling}
Ahmed El~Alaoui, Andrea Montanari, and Mark Sellke.
\newblock Sampling from mean-field gibbs measures via diffusion processes.
\newblock \emph{arXiv preprint arXiv:2310.08912}, 2023.

\bibitem[Fan et~al.(2023)Fan, Watkins, Du, Liu, Ryu, Boutilier, Abbeel, Ghavamzadeh, Lee, and Lee]{fan2023dpok}
Ying Fan, Olivia Watkins, Yuqing Du, Hao Liu, Moonkyung Ryu, Craig Boutilier, Pieter Abbeel, Mohammad Ghavamzadeh, Kangwook Lee, and Kimin Lee.
\newblock Dpok: Reinforcement learning for fine-tuning text-to-image diffusion models.
\newblock \emph{arXiv preprint arXiv:2305.16381}, 2023.

\bibitem[Fu et~al.(2024)Fu, Yang, Wang, and Chen]{fu2024unveil}
Hengyu Fu, Zhuoran Yang, Mengdi Wang, and Minshuo Chen.
\newblock Unveil conditional diffusion models with classifier-free guidance: A sharp statistical theory.
\newblock \emph{arXiv preprint arXiv:2403.11968}, 2024.

\bibitem[Garber and Tirer(2024)]{garber2024image}
Tomer Garber and Tom Tirer.
\newblock Image restoration by denoising diffusion models with iteratively preconditioned guidance.
\newblock In \emph{Proceedings of the IEEE/CVF Conference on Computer Vision and Pattern Recognition}, pages 25245--25254, 2024.

\bibitem[Gruver et~al.(2023)Gruver, Stanton, Frey, Rudner, Hotzel, Lafrance-Vanasse, Rajpal, Cho, and Wilson]{gruver2023protein}
Nate Gruver, Samuel Stanton, Nathan~C Frey, Tim~GJ Rudner, Isidro Hotzel, Julien Lafrance-Vanasse, Arvind Rajpal, Kyunghyun Cho, and Andrew~Gordon Wilson.
\newblock Protein design with guided discrete diffusion.
\newblock \emph{arXiv preprint arXiv:2305.20009}, 2023.

\bibitem[Guo et~al.(2023)Guo, Liu, Wang, Chen, Wang, Xu, and Cheng]{guo2023diffusion}
Zhiye Guo, Jian Liu, Yanli Wang, Mengrui Chen, Duolin Wang, Dong Xu, and Jianlin Cheng.
\newblock Diffusion models in bioinformatics: A new wave of deep learning revolution in action.
\newblock \emph{arXiv preprint arXiv:2302.10907}, 2023.

\bibitem[He et~al.(2016)He, Zhang, Ren, and Sun]{he2016deep}
Kaiming He, Xiangyu Zhang, Shaoqing Ren, and Jian Sun.
\newblock Deep residual learning for image recognition.
\newblock In \emph{Proceedings of the IEEE conference on computer vision and pattern recognition}, pages 770--778, 2016.

\bibitem[He et~al.(2023)He, Murata, Lai, Takida, Uesaka, Kim, Liao, Mitsufuji, Kolter, Salakhutdinov, et~al.]{he2023manifold}
Yutong He, Naoki Murata, Chieh-Hsin Lai, Yuhta Takida, Toshimitsu Uesaka, Dongjun Kim, Wei-Hsiang Liao, Yuki Mitsufuji, J~Zico Kolter, Ruslan Salakhutdinov, et~al.
\newblock Manifold preserving guided diffusion.
\newblock \emph{arXiv preprint arXiv:2311.16424}, 2023.

\bibitem[Ho and Salimans(2022)]{ho2022classifier}
Jonathan Ho and Tim Salimans.
\newblock Classifier-free diffusion guidance.
\newblock \emph{arXiv preprint arXiv:2207.12598}, 2022.

\bibitem[Ho et~al.(2020)Ho, Jain, and Abbeel]{ho2020denoising}
Jonathan Ho, Ajay Jain, and Pieter Abbeel.
\newblock Denoising diffusion probabilistic models.
\newblock \emph{Advances in neural information processing systems}, 33:\penalty0 6840--6851, 2020.

\bibitem[Karunratanakul et~al.(2024)Karunratanakul, Preechakul, Aksan, Beeler, Suwajanakorn, and Tang]{karunratanakul2024optimizing}
Korrawe Karunratanakul, Konpat Preechakul, Emre Aksan, Thabo Beeler, Supasorn Suwajanakorn, and Siyu Tang.
\newblock Optimizing diffusion noise can serve as universal motion priors.
\newblock In \emph{Proceedings of the IEEE/CVF Conference on Computer Vision and Pattern Recognition}, pages 1334--1345, 2024.

\bibitem[Kawar et~al.(2022)Kawar, Elad, Ermon, and Song]{kawar2022denoising}
Bahjat Kawar, Michael Elad, Stefano Ermon, and Jiaming Song.
\newblock Denoising diffusion restoration models.
\newblock \emph{Advances in Neural Information Processing Systems}, 35:\penalty0 23593--23606, 2022.

\bibitem[Kong et~al.(2020)Kong, Ping, Huang, Zhao, and Catanzaro]{kong2020diffwave}
Zhifeng Kong, Wei Ping, Jiaji Huang, Kexin Zhao, and Bryan Catanzaro.
\newblock Diffwave: A versatile diffusion model for audio synthesis.
\newblock \emph{arXiv preprint arXiv:2009.09761}, 2020.

\bibitem[Lee et~al.(2022{\natexlab{a}})Lee, Lu, and Tan]{lee2022convergencea}
Holden Lee, Jianfeng Lu, and Yixin Tan.
\newblock Convergence for score-based generative modeling with polynomial complexity.
\newblock \emph{arXiv preprint arXiv:2206.06227}, 2022{\natexlab{a}}.

\bibitem[Lee et~al.(2022{\natexlab{b}})Lee, Lu, and Tan]{lee2022convergenceb}
Holden Lee, Jianfeng Lu, and Yixin Tan.
\newblock Convergence of score-based generative modeling for general data distributions.
\newblock \emph{arXiv preprint arXiv:2209.12381}, 2022{\natexlab{b}}.

\bibitem[Liu et~al.(2022)Liu, Wu, Ye, and Liu]{liu2022let}
Xingchao Liu, Lemeng Wu, Mao Ye, and Qiang Liu.
\newblock Let us build bridges: Understanding and extending diffusion generative models.
\newblock \emph{arXiv preprint arXiv:2208.14699}, 2022.

\bibitem[Liu et~al.(2024)Liu, Zhang, Li, Yan, Gao, Chen, Yuan, Huang, Sun, Gao, et~al.]{liu2024sora}
Yixin Liu, Kai Zhang, Yuan Li, Zhiling Yan, Chujie Gao, Ruoxi Chen, Zhengqing Yuan, Yue Huang, Hanchi Sun, Jianfeng Gao, et~al.
\newblock Sora: A review on background, technology, limitations, and opportunities of large vision models.
\newblock \emph{arXiv preprint arXiv:2402.17177}, 2024.

\bibitem[Lugmayr et~al.(2022)Lugmayr, Danelljan, Romero, Yu, Timofte, and Van~Gool]{lugmayr2022repaint}
Andreas Lugmayr, Martin Danelljan, Andres Romero, Fisher Yu, Radu Timofte, and Luc Van~Gool.
\newblock Repaint: Inpainting using denoising diffusion probabilistic models.
\newblock In \emph{Proceedings of the IEEE/CVF Conference on Computer Vision and Pattern Recognition}, pages 11461--11471, 2022.

\bibitem[Marion et~al.(2024)Marion, Korba, Bartlett, Blondel, De~Bortoli, Doucet, Llinares-L{\'o}pez, Paquette, and Berthet]{marion2024implicit}
Pierre Marion, Anna Korba, Peter Bartlett, Mathieu Blondel, Valentin De~Bortoli, Arnaud Doucet, Felipe Llinares-L{\'o}pez, Courtney Paquette, and Quentin Berthet.
\newblock Implicit diffusion: Efficient optimization through stochastic sampling.
\newblock \emph{arXiv preprint arXiv:2402.05468}, 2024.

\bibitem[Mei and Wu(2023)]{mei2023deep}
Song Mei and Yuchen Wu.
\newblock Deep networks as denoising algorithms: Sample-efficient learning of diffusion models in high-dimensional graphical models.
\newblock \emph{arXiv preprint arXiv:2309.11420}, 2023.

\bibitem[Montanari and Wu(2023)]{montanari2023posterior}
Andrea Montanari and Yuchen Wu.
\newblock Posterior sampling from the spiked models via diffusion processes.
\newblock \emph{arXiv preprint arXiv:2304.11449}, 2023.

\bibitem[Nichol and Dhariwal(2021)]{nichol2021improved}
Alexander~Quinn Nichol and Prafulla Dhariwal.
\newblock Improved denoising diffusion probabilistic models.
\newblock In \emph{Proceedings of the International Conference on Machine Learning}, pages 8162--8171. PMLR, 2021.

\bibitem[Oko et~al.(2023)Oko, Akiyama, and Suzuki]{oko2023diffusion}
Kazusato Oko, Shunta Akiyama, and Taiji Suzuki.
\newblock Diffusion models are minimax optimal distribution estimators.
\newblock \emph{arXiv preprint arXiv:2303.01861}, 2023.

\bibitem[Pan et~al.(2023{\natexlab{a}})Pan, Liew, Tan, Feng, and Yan]{pan2023adjointdpm}
Jiachun Pan, Jun~Hao Liew, Vincent~YF Tan, Jiashi Feng, and Hanshu Yan.
\newblock Adjointdpm: Adjoint sensitivity method for gradient backpropagation of diffusion probabilistic models.
\newblock \emph{arXiv preprint arXiv:2307.10711}, 2023{\natexlab{a}}.

\bibitem[Pan et~al.(2023{\natexlab{b}})Pan, Yan, Liew, Feng, and Tan]{pan2023towards}
Jiachun Pan, Hanshu Yan, Jun~Hao Liew, Jiashi Feng, and Vincent~YF Tan.
\newblock Towards accurate guided diffusion sampling through symplectic adjoint method.
\newblock \emph{arXiv preprint arXiv:2312.12030}, 2023{\natexlab{b}}.

\bibitem[Pope et~al.(2021)Pope, Zhu, Abdelkader, Goldblum, and Goldstein]{pope2021intrinsic}
Phillip Pope, Chen Zhu, Ahmed Abdelkader, Micah Goldblum, and Tom Goldstein.
\newblock The intrinsic dimension of images and its impact on learning.
\newblock \emph{arXiv preprint arXiv:2104.08894}, 2021.

\bibitem[Prabhudesai et~al.(2023)Prabhudesai, Goyal, Pathak, and Fragkiadaki]{prabhudesai2023aligning}
Mihir Prabhudesai, Anirudh Goyal, Deepak Pathak, and Katerina Fragkiadaki.
\newblock Aligning text-to-image diffusion models with reward backpropagation.
\newblock \emph{arXiv preprint arXiv:2310.03739}, 2023.

\bibitem[Rombach et~al.(2022)Rombach, Blattmann, Lorenz, Esser, and Ommer]{rombach2022high}
Robin Rombach, Andreas Blattmann, Dominik Lorenz, Patrick Esser, and Bj{\"o}rn Ommer.
\newblock High-resolution image synthesis with latent diffusion models.
\newblock In \emph{Proceedings of the IEEE/CVF Conference on Computer Vision and Pattern Recognition}, pages 10684--10695, 2022.

\bibitem[Ronneberger et~al.(2015)Ronneberger, Fischer, and Brox]{ronneberger2015u}
Olaf Ronneberger, Philipp Fischer, and Thomas Brox.
\newblock U-net: Convolutional networks for biomedical image segmentation.
\newblock In \emph{Medical Image Computing and Computer-Assisted Intervention--MICCAI 2015: 18th International Conference, Munich, Germany, October 5-9, 2015, Proceedings, Part III 18}, pages 234--241. Springer, 2015.

\bibitem[Roweis and Saul(2000)]{roweis2000nonlinear}
Sam~T Roweis and Lawrence~K Saul.
\newblock Nonlinear dimensionality reduction by locally linear embedding.
\newblock \emph{science}, 290\penalty0 (5500):\penalty0 2323--2326, 2000.

\bibitem[Song et~al.(2020{\natexlab{a}})Song, Meng, and Ermon]{song2020denoising}
Jiaming Song, Chenlin Meng, and Stefano Ermon.
\newblock Denoising diffusion implicit models.
\newblock \emph{arXiv preprint arXiv:2010.02502}, 2020{\natexlab{a}}.

\bibitem[Song et~al.(2023)Song, Zhang, Yin, Mardani, Liu, Kautz, Chen, and Vahdat]{song2023loss}
Jiaming Song, Qinsheng Zhang, Hongxu Yin, Morteza Mardani, Ming-Yu Liu, Jan Kautz, Yongxin Chen, and Arash Vahdat.
\newblock Loss-guided diffusion models for plug-and-play controllable generation.
\newblock In \emph{International Conference on Machine Learning}, pages 32483--32498. PMLR, 2023.

\bibitem[Song and Ermon(2019)]{song2019generative}
Yang Song and Stefano Ermon.
\newblock Generative modeling by estimating gradients of the data distribution.
\newblock \emph{Advances in Neural Information Processing Systems}, 32, 2019.

\bibitem[Song et~al.(2020{\natexlab{b}})Song, Garg, Shi, and Ermon]{song2020sliced}
Yang Song, Sahaj Garg, Jiaxin Shi, and Stefano Ermon.
\newblock Sliced score matching: A scalable approach to density and score estimation.
\newblock In \emph{Uncertainty in Artificial Intelligence}, pages 574--584. PMLR, 2020{\natexlab{b}}.

\bibitem[Song et~al.(2020{\natexlab{c}})Song, Sohl-Dickstein, Kingma, Kumar, Ermon, and Poole]{song2020score}
Yang Song, Jascha Sohl-Dickstein, Diederik~P Kingma, Abhishek Kumar, Stefano Ermon, and Ben Poole.
\newblock Score-based generative modeling through stochastic differential equations.
\newblock \emph{arXiv preprint arXiv:2011.13456}, 2020{\natexlab{c}}.

\bibitem[Tang et~al.(2024)Tang, Peng, Tang, Hong, Wang, and Chang]{tang2024tuning}
Zhiwei Tang, Jiangweizhi Peng, Jiasheng Tang, Mingyi Hong, Fan Wang, and Tsung-Hui Chang.
\newblock Tuning-free alignment of diffusion models with direct noise optimization.
\newblock \emph{arXiv preprint arXiv:2405.18881}, 2024.

\bibitem[Tenenbaum et~al.(2000)Tenenbaum, Silva, and Langford]{tenenbaum2000global}
Joshua~B Tenenbaum, Vin~de Silva, and John~C Langford.
\newblock A global geometric framework for nonlinear dimensionality reduction.
\newblock \emph{science}, 290\penalty0 (5500):\penalty0 2319--2323, 2000.

\bibitem[Uehara et~al.(2024)Uehara, Zhao, Black, Hajiramezanali, Scalia, Diamant, Tseng, Biancalani, and Levine]{uehara2024fine}
Masatoshi Uehara, Yulai Zhao, Kevin Black, Ehsan Hajiramezanali, Gabriele Scalia, Nathaniel~Lee Diamant, Alex~M Tseng, Tommaso Biancalani, and Sergey Levine.
\newblock Fine-tuning of continuous-time diffusion models as entropy-regularized control.
\newblock \emph{arXiv preprint arXiv:2402.15194}, 2024.

\bibitem[Wallace et~al.(2023)Wallace, Gokul, Ermon, and Naik]{wallace2023end}
Bram Wallace, Akash Gokul, Stefano Ermon, and Nikhil Naik.
\newblock End-to-end diffusion latent optimization improves classifier guidance.
\newblock In \emph{Proceedings of the IEEE/CVF International Conference on Computer Vision}, pages 7280--7290, 2023.

\bibitem[Wang et~al.(2022)Wang, Yu, and Zhang]{wang2022zero}
Yinhuai Wang, Jiwen Yu, and Jian Zhang.
\newblock Zero-shot image restoration using denoising diffusion null-space model.
\newblock \emph{arXiv preprint arXiv:2212.00490}, 2022.

\bibitem[Watson et~al.(2023)Watson, Juergens, Bennett, Trippe, Yim, Eisenach, Ahern, Borst, Ragotte, Milles, et~al.]{watson2023novo}
Joseph~L Watson, David Juergens, Nathaniel~R Bennett, Brian~L Trippe, Jason Yim, Helen~E Eisenach, Woody Ahern, Andrew~J Borst, Robert~J Ragotte, Lukas~F Milles, et~al.
\newblock De novo design of protein structure and function with rfdiffusion.
\newblock \emph{Nature}, 620\penalty0 (7976):\penalty0 1089--1100, 2023.

\bibitem[Wibisono et~al.(2024)Wibisono, Wu, and Yang]{wibisono2024optimal}
Andre Wibisono, Yihong Wu, and Kaylee~Yingxi Yang.
\newblock Optimal score estimation via empirical bayes smoothing.
\newblock \emph{arXiv preprint arXiv:2402.07747}, 2024.

\bibitem[Xu et~al.(2023)Xu, Liu, Wu, Tong, Li, Ding, Tang, and Dong]{xu2023imagereward}
Jiazheng Xu, Xiao Liu, Yuchen Wu, Yuxuan Tong, Qinkai Li, Ming Ding, Jie Tang, and Yuxiao Dong.
\newblock Imagereward: Learning and evaluating human preferences for text-to-image generation.
\newblock \emph{arXiv preprint arXiv:2304.05977}, 2023.

\bibitem[Yang et~al.(2023)Yang, Zhang, Song, Hong, Xu, Zhao, Zhang, Cui, and Yang]{yang2023diffusion}
Ling Yang, Zhilong Zhang, Yang Song, Shenda Hong, Runsheng Xu, Yue Zhao, Wentao Zhang, Bin Cui, and Ming-Hsuan Yang.
\newblock Diffusion models: A comprehensive survey of methods and applications.
\newblock \emph{ACM Computing Surveys}, 56\penalty0 (4):\penalty0 1--39, 2023.

\bibitem[Yu et~al.(2023)Yu, Wang, Zhao, Ghanem, and Zhang]{yu2023freedom}
Jiwen Yu, Yinhuai Wang, Chen Zhao, Bernard Ghanem, and Jian Zhang.
\newblock Freedom: Training-free energy-guided conditional diffusion model.
\newblock In \emph{Proceedings of the IEEE/CVF International Conference on Computer Vision}, pages 23174--23184, 2023.

\bibitem[Yuan et~al.(2023)Yuan, Huang, Ni, Chen, and Wang]{yuan2023reward}
Hui Yuan, Kaixuan Huang, Chengzhuo Ni, Minshuo Chen, and Mengdi Wang.
\newblock Reward-directed conditional diffusion: Provable distribution estimation and reward improvement.
\newblock \emph{arXiv preprint arXiv:2307.07055}, 2023.

\end{thebibliography}
\end{document}